\documentclass[11pt,letterpaper,logo]{style}

\usepackage[numbers]{natbib}
\usepackage{graphicx}
\usepackage{booktabs}
\usepackage{amsmath,amsfonts,amssymb}
\usepackage{subcaption}
\usepackage{multirow}
\usepackage{colortbl}
\usepackage{listings}
\usepackage{xparse}
\usepackage{fontawesome5}
\usepackage{threeparttable}

\graphicspath{{./}{fig/}{figures/}{plot/}{pdf/}{table/}}
\usepackage{amsthm}
\usepackage{tcolorbox}
\tcbuselibrary{skins,breakable}
\tcbuselibrary{listingsutf8}
\usepackage{titletoc}
\usepackage{pifont}
\usepackage{mathtools}
\usepackage{bbm}
\usepackage{makecell}
\usepackage{adjustbox}
\newtheorem{proposition}{Proposition}

\usepackage{amsmath,amsfonts,bm}

\def\eqref#1{equation~\ref{#1}}

\def\1{\bm{1}}

\DeclareMathAlphabet{\mathsfit}{\encodingdefault}{\sfdefault}{m}{sl}
\SetMathAlphabet{\mathsfit}{bold}{\encodingdefault}{\sfdefault}{bx}{n}

\usepackage{siunitx}
\usepackage{algorithm}
\usepackage{algpseudocode}
\usepackage{wrapfig}
\definecolor{PaperBlue}{HTML}{224d70}
\definecolor{PaperRed}{HTML}{872b26}
\definecolor{PaperLightBlue}{HTML}{8eb6cd}
\definecolor{PaperGray}{HTML}{e8edf2}
\colorlet{Blue}{PaperBlue}
\colorlet{Red}{PaperRed}
\colorlet{Gray}{PaperGray}

\definecolor{TableGroup}{HTML}{F1F3F5}
\definecolor{TableHighlight}{HTML}{E8F0F7}
\definecolor{TableAccent}{HTML}{224D70}
\definecolor{TablePositive}{HTML}{EAF4EE}
\newcommand{\tablefont}{\fontsize{8.2}{9.8}\selectfont}
\newcommand{\tablesetup}{\tablefont\setlength{\tabcolsep}{3.6pt}\renewcommand{\arraystretch}{1.12}}
\newcommand{\methodtag}[1]{\textcolor{TableAccent}{\textbf{#1}}}

\newcommand{\CustomTitleLayout}{%
  \noindent\begin{minipage}{\linewidth}
    {\titlefont\bfseries\fontsize{29}{31}\selectfont
      \raisebox{-0.17\height}{\includegraphics[height=0.50in]{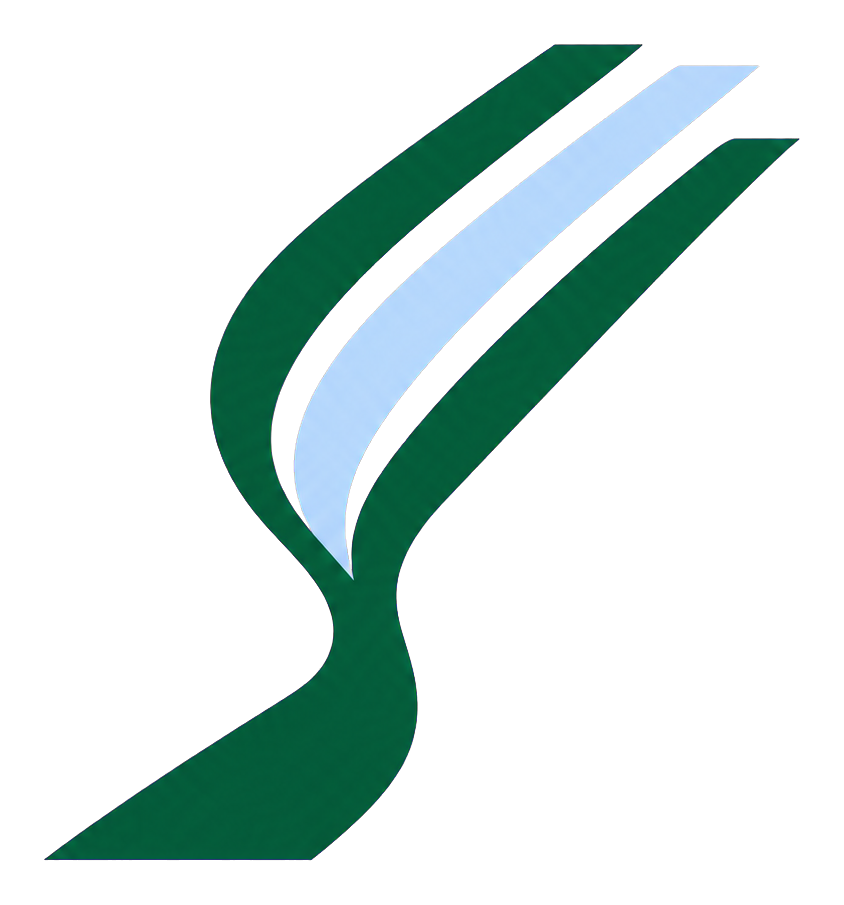}}%
      \hspace{0.08in}\textcolor{ThemeBorder}{SLPO}\par}
    \vspace{1pt}
    {\titlefont\bfseries\fontsize{22}{25}\selectfont
      Scaling Latent Reasoning via a Surrogate Policy\par}
  \end{minipage}%
}

\title{\textcolor{ThemeBorder}{SLPO}\, Scaling Latent Reasoning via a Surrogate Policy}
\runningtitle{SLPO: Scaling Latent Reasoning via a Surrogate Policy}
\PublicDate{2026-07-20}

\author{%
  {\Authfont
    \textbf{Runyang You}\textsuperscript{1}\equal \quad
    \textbf{Zhiyuan Liu}\textsuperscript{2}\equal\intern \quad
    \textbf{Yongqi Li}\textsuperscript{1}\advisor \quad
    \textbf{Wenjie Li}\textsuperscript{1}
  }\\
  {\Affilfont
    \textsuperscript{1} The Hong Kong Polytechnic University \quad
    \textsuperscript{2} Sichuan University \\
    \texttt{runyang.y@outlook.com, liyongqi0@gmail.com}
  }
}


\begin{document}

\begin{abstract}
Reinforcement learning with verifiable rewards has become the predominant recipe
for eliciting test-time scaling in explicit Chain-of-Thought reasoners.
Yet this scaling path remains computationally costly,
since intermediate reasoning must be externalized as natural-language tokens.
Latent reasoning instead carries intermediate computation as continuous vectors
and already matches or surpasses explicit CoT at far shorter horizons.
Despite this promise,
latent reasoners remain largely imitation-bound,
while explicit CoT has already moved past imitation via outcome-reward RL.
Latent trajectories lack a tractable per-step likelihood
and an adaptive stopping interface under fixed thinking budgets,
so outcome rewards cannot elicit latent test-time scaling.
We introduce Surrogate Latent Policy Optimization (SLPO)
to bring outcome-reward RL to autoregressive latent reasoners:
a differentiable surrogate policy interface over latent transitions
for trajectory-level credit assignment,
and a correctness-supervised stopping head
that outcome-reward optimization refines into a variable-horizon policy.
Across two continuous latent reasoners,
two backbones,
and three held-out benchmarks,
SLPO improves Pass@8 and Pass@16 in all 12 backbone--dataset settings,
with gains of up to $12.07$ percentage points.
SLPO further transfers to soft-token inference
and learns difficulty-adaptive computation,
allocating longer latent trajectories to harder instances.
\end{abstract}

\newcommand{\TitleLinks}{%
\centering
    \vspace{6pt}
    {\noindent\absfont\fontsize{11}{13}\selectfont
    \faGithub\ Project Page: \url{https://modalitydance.github.io/SLPO/}\par}%
}
\maketitle

\begin{figure*}[h]
    \centering
    \begin{subfigure}[t]{0.325\textwidth}
        \centering
        \includegraphics[width=\linewidth]{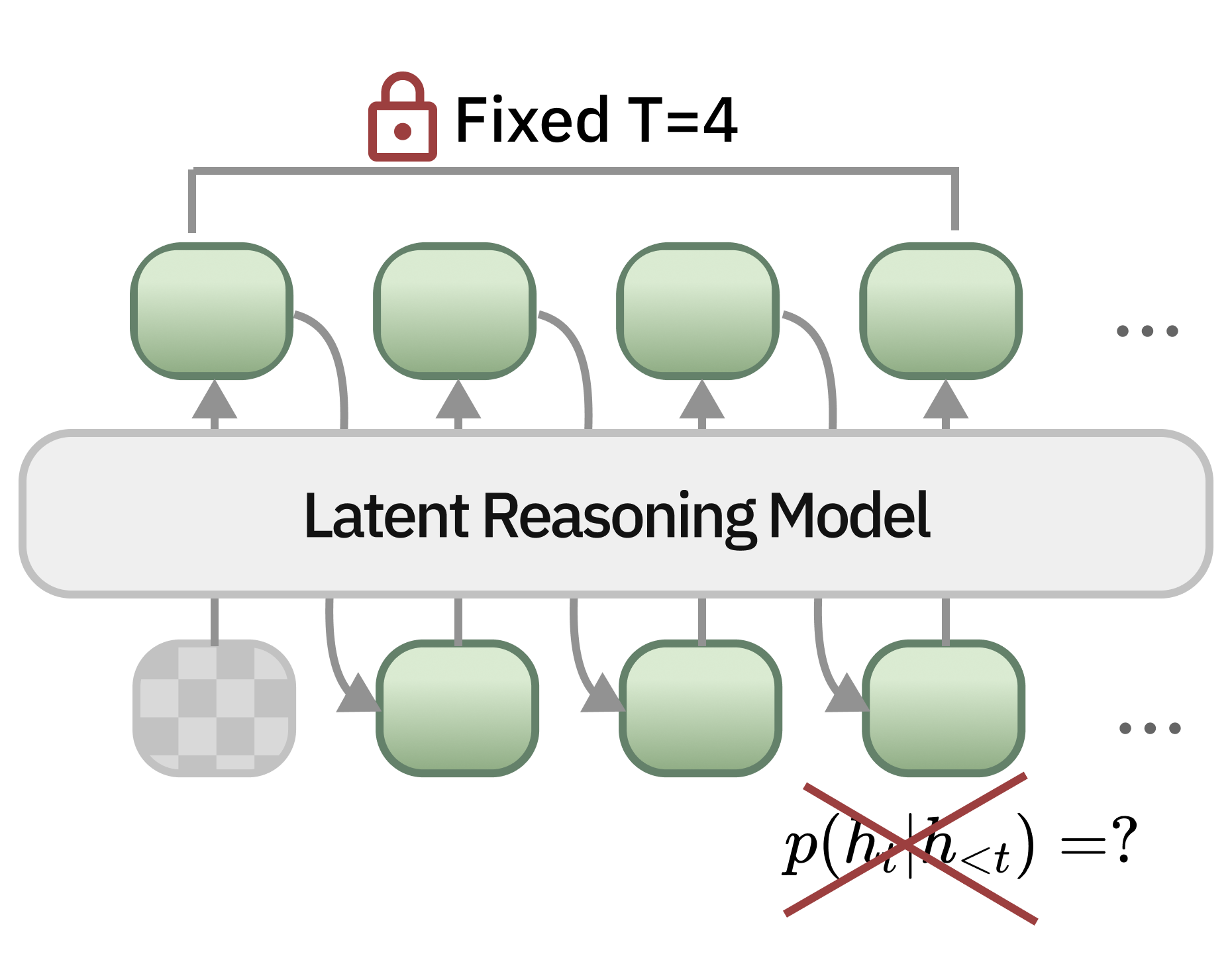}
      \caption{Latent Reasoning}
      \label{fig:teaser-bottleneck}
  \end{subfigure}
  \hfill
  \begin{subfigure}[t]{0.325\textwidth}
      \centering
      \includegraphics[width=\linewidth]{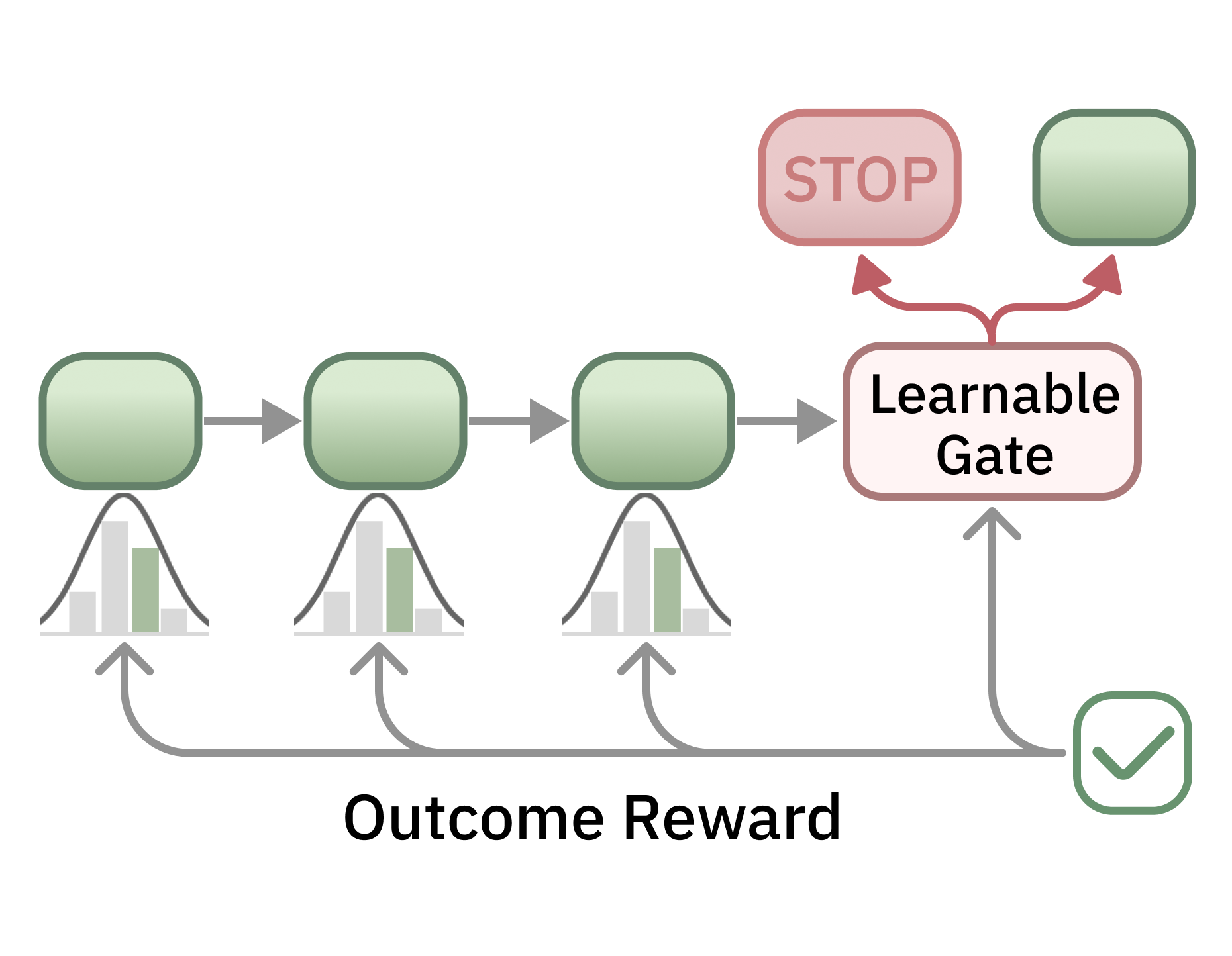}
      \caption{SLPO}
      \label{fig:teaser-slpo}
  \end{subfigure}
    \hfill
    \begin{subfigure}[t]{0.325\textwidth}
        \centering
        \includegraphics[width=\linewidth]{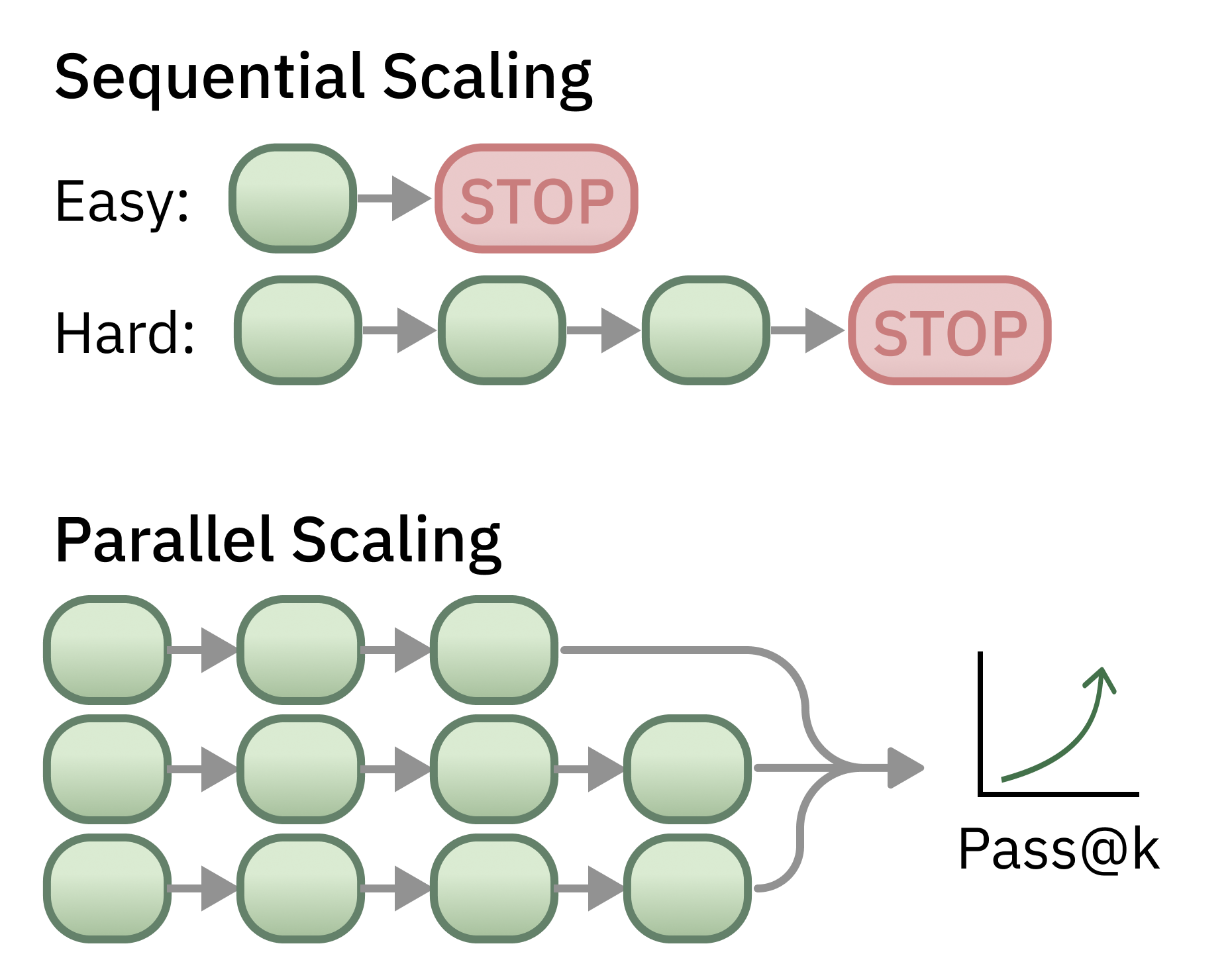}
        \caption{Latent Test-Time Scaling}
        \label{fig:teaser-scaling}
    \end{subfigure}
    \caption{SLPO brings outcome-reward RL to latent reasoning and enables latent test-time scaling.}
    \label{fig:teaser}
\end{figure*}


\section{Introduction}


Large language models have demonstrated strong performance on complex reasoning tasks
with chain-of-thought (CoT) reasoning~\citep{cot,self-consistency}.
By eliciting intermediate steps in natural language,
CoT decomposes difficult problems and allocates additional test-time computation before generating the final answer~\citep{s1,snell-tts-optimal}.
Recent reasoning models scale this additional test-time computation considerably further,
either sequentially, by extending and revising ever longer chains~\citep{deepseek-r1};
or in parallel, by sampling and searching among ever more trials at inference~\citep{self-consistency,snell-tts-optimal}.
These inference-time strategies have become known as test-time scaling (TTS), while reinforcement learning with verifiable rewards (RLVR) has emerged as the predominant training-time paradigm for eliciting and optimizing such capabilities~\citep{rlvr-incentivizes,rl-lrm-survey}.

However, this test-time scaling path is computationally costly.
Explicit CoT externalizes intermediate reasoning as natural-language tokens,
but many of these tokens support linguistic coherence,
exposition,
or redundancy rather than the underlying problem-solving state~\citep{efficient-reasoning-survey,regular}.
Each token still incurs the full cost of autoregressive generation,
fundamentally limiting the scalability of explicit reasoning.
This motivates \emph{continuous CoT}, or \emph{latent reasoning}. 
As illustrated in Fig.~\ref{fig:teaser-bottleneck},
latent reasoning carries intermediate computation as continuous vectors in hidden space rather than as discrete language tokens~\citep{coconut}.
Recent latent-reasoning models already match or surpass explicit CoT on standard benchmarks with far shorter intermediate computation~\citep{latent-cot-survey,codi,colar,sim-cot,rot,regular}.

Despite the promise,
latent reasoners remain imitation-bound:
they are largely trained to align their intermediate states with manually-selected~\citep{coconut, codi}, compressed~\citep{colar, sim-cot}, or rendered~\citep{rot,regular} explicit CoT representations,
reproducing predefined trajectories rather than optimizing latent computation for task outcomes.
Explicit CoT has already moved beyond this imitation stage:
RLVR~\citep{deepseek-r1,rlvr-incentivizes} is now a mature recipe for scaling past teacher imitation in explicit reasoning.
Rollout length is inherently variable,
allowing candidate trajectories to span different reasoning horizons.
Because each reasoning step is a token sampled from the vocabulary distribution,
policy gradients can assign trajectory-level credit through a tractable per-step likelihood,
so outcome rewards can reshape both the reasoning content and test-time compute allocation.
However, latent reasoning has not yet reached an equivalent outcome-optimization stage.
Two obstacles prevent this extension:
(1) Existing latent reasoners prescribe a fixed thinking budget throughout training and inference~\citep{coconut,codi},
freezing the compute horizon that RL would otherwise optimize.
(2) As shown in Fig.~\ref{fig:teaser-bottleneck},
intermediate steps propagate as continuous vectors that bypass the vocabulary distribution,
so latent trajectories lack a tractable action likelihood through which outcome rewards can assign transition-level credit~\citep{rlvr-limit,reasoning-boundary-paradox}.
In sum, absent both an adaptive stopping interface and a tractable latent-transition likelihood,
outcome rewards cannot elicit latent test-time scaling.

In this work,
we introduce {\textbf{S}urrogate \textbf{L}atent \textbf{P}olicy \textbf{O}ptimization (SLPO)},
which instantiates both interfaces to bring outcome-level RL to latent reasoning,
as illustrated by Fig.~\ref{fig:teaser-slpo}.
To make the rollout horizon learnable,
we equip the reasoner with a stopping head
and initialize it through a correctness-supervised cold start,
providing an initial stopping policy that outcome-reward optimization subsequently refines
for adaptive computation.
To optimize the resulting trajectories,
SLPO defines a surrogate likelihood in hidden space,
allowing policy-gradient estimators to assign credit to vector-based reasoning without vocabulary-level action probabilities.
Together, these components enable outcome rewards to jointly optimize latent transitions and stopping decisions
in autoregressive latent reasoners,
using standard algorithms such as RLOO or GRPO.
The resulting policies scale latent computation with problem difficulty and improve held-out accuracy
(Fig.~\ref{fig:teaser-scaling}).

We summarize our main contributions as follows:

\begin{itemize}
      \item We identify two barriers to outcome-reward optimization in latent reasoning:
            fixed thinking budgets constrain candidate trajectories to a prescribed horizon,
            while continuous latent transitions lack a tractable policy likelihood.

      \item We introduce SLPO,
            equipping latent reasoners with a learnable stopping policy for variable-length rollouts
            and a surrogate likelihood for credit assignment over continuous latent transitions,
            enabling outcome rewards to jointly optimize latent reasoning and its compute allocation.

      \item Across two latent reasoners,
            two backbones,
            and three held-out benchmarks,
            SLPO improves both Pass@8 and Pass@16 in all 12 backbone--dataset settings,
            with gains of up to $12.07$ percentage points.
            The gains persist across RLOO and GRPO,
            transfer to soft-token inference,
            and accompany difficulty-adaptive latent computation.
      \end{itemize}

\section{Related Work}

\paragraph{Latent reasoning.}
Latent reasoning instead autoregresses in hidden space,
compressing redundant explicit traces into continuous thoughts~\citep{efficient-reasoning-survey,latent-cot-survey,coconut,codi,colar,sim-cot,rot,regular,dart,latent-sft}.
These approaches compress reasoning through curriculum training, self-distillation,
dynamic or variational latent traces, and vocabulary-space representations~\citep{coconut,codi,colar,regular,dart,latent-sft,sim-cot,rot}.
These methods optimize latent states under SFT-style supervision from explicit reasoning,
typically with a fixed inference budget; parallel test-time scaling instead keeps weights frozen,
stochastically rolls out latent trajectories, and aggregates them with a learned reward model~\citep{parallel-latent-tts}.

\paragraph{Reinforcement learning.}
Outcome-reward RL has become central to scaling explicit CoT~\citep{rl-lrm-survey,deepseek-r1,rlvr-incentivizes}.
Token-policy objectives such as GRPO and RLOO remain standard~\citep{deepseekmath-grpo,effective-rl-reasoning},
but continuous hidden transitions expose no action log-probabilities~\citep{rlvr-limit,reasoning-boundary-paradox}.
Recent methods address this gap with soft-token rollouts, vocabulary-space latent tokens, or learned latent heads~\citep{lepo,latent-grpo,colar}.
These designs depend on token distributions, embedding mixtures, or architecture-specific latent heads,
so they do not carry over to arbitrary vector-based latent reasoners with plain hidden-state recurrence.
Outcome-reward optimization in this setting remains unexplored.

\section{Preliminaries}

We begin by formalizing RL over explicit CoT and reviewing latent CoT,
and outline dropout-induced stochastic latent rollouts that underlie the methodology below.
For each training instance $(x_i,a_i^\star)$,
$x_i$ denotes the input problem and $a_i^\star$ denotes the reference answer.

\paragraph{Explicit CoT and RL.}
In explicit CoT,
a model samples a discrete reasoning trace autoregressively from the next-token distribution~\citep{cot,self-consistency}.
At reasoning step $t$,
$y_t\sim\pi_\theta(\cdot\mid x_i,y_{<t})$;
after the reasoning trace is sampled,
the model generates the final answer tokens conditioned on the trace.
The resulting rollout likelihood is therefore available directly from token probabilities:
\[
\log \pi_\theta(y_{1:T},\hat{a}_i\mid x_i)
=
\sum_{t=1}^{T}
\log \pi_\theta(y_t\mid x_i,y_{<t})
+
\sum_{s=1}^{|\hat{a}_i|}
\log \pi_\theta(\hat{a}_{i,s}\mid x_i,y_{1:T},\hat{a}_{i,<s}).
\]
Given a verifiable reward $R(\hat{a}_i,a_i^\star)$,
policy-gradient methods assign credit to sampled rollouts
by weighting this log-probability term with an advantage such as
$\widehat{A}_i=R(\hat{a}_i,a_i^\star)-b_i$.
Thus explicit CoT exposes the core policy interface needed by RL:
the sampled reasoning trajectory already comes with token-level action probabilities.

\paragraph{Latent CoT.}
Continuous CoT replaces the discrete reasoning trace with an autoregressive trajectory in hidden space~\citep{coconut,latent-cot-survey}.
We write this trajectory as
$h_{i,1:T_{\max}}=(h_{i,1},\ldots,h_{i,T_{\max}})$,
where $h_{i,t}\in\mathbb{R}^d$ is the latent state at step $t$.
The latent update can be written abstractly as
\[
h_{i,t}
=
f_\theta(x_i,h_{i,<t}),
\]
with the answer decoder conditioning on the completed latent trajectory.
During latent reasoning,
the intermediate computation bypasses the vocabulary distribution and language-head sampling,
which removes the token-policy likelihood that explicit CoT exposes to RL.
Prior work has shown that latent reasoning can nevertheless be made stochastic
by inference under dropout~\citep{parallel-latent-tts}.
Such stochasticity supplies the exploration substrate for RL,
and we denote a sampled trajectory by $\tilde{h}_{i,1:T_{\max}}$.
However, how to construct a tractable log-likelihood for the realized latent states is still an open question.

\section{Methodology}

Existing latent reasoners typically prescribe a fixed thinking budget,
so the model cannot decide when to terminate latent computation on its own.
We therefore proceed in two stages.
We first attach a stopping head and, for each sampled trajectory,
supervise its distribution over stopping times with final-answer correctness evaluated at every candidate length (Sec.~\ref{sec:stop-gate-cold-start}).
This introduces a learnable continue-versus-stop decision before outcome-reward optimization.
SLPO then makes outcome rewards usable for eliciting latent reasoning scaling:
a tractable surrogate likelihood scores continuous latent transitions,
and the stopping likelihood is included in the policy objective,
so verifiable rewards jointly shape latent reasoning and compute allocation (Sec.~\ref{sec:slpo}).
We close with the inference procedure (Sec.~\ref{sec:inference}).

\subsection{Stopping-gate Cold Start}
\label{sec:stop-gate-cold-start}
We augment the latent reasoner with a stopping head $g_\theta$ over latent states.
Given a latent state $h$,
the head outputs
\[
s_\theta(h)
=
\sigma(g_\theta(h)),
\]
which is the probability of terminating latent computation at that state.
For an input $x_i$,
the released backbone samples $N$ stochastic latent trajectories up to a maximum budget:
\[
h^{(n)}_{i,1:T_{\max}}
\sim
\pi_\theta(\cdot\mid x_i),
\qquad
n=1,\ldots,N.
\]
For each sampled trajectory,
we enumerate candidate stopping lengths $t\in[T_{\min},T_{\max}]$.
At each candidate length,
we halt latent thinking after $h^{(n)}_{i,t}$
and decode final-answers explicitly conditioned on $x_i$ and the resulting prefix:
\[
a_{i,n,t}
\sim
\pi_\theta(\cdot\mid x_i,h^{(n)}_{i,1:t}).
\]
Using a verifiable reward,
the answer-valid stopping set is
\[
\mathcal{V}^{(n)}_i
=
\left\{
t\in[T_{\min},T_{\max}]
\mid
R(a_{i,n,t},a_i^\star)=1
\right\}.
\]

For a sampled trajectory,
write $\rho_{i,n,t}=s_\theta(h^{(n)}_{i,t})$ as the stop probability at step $t$.
The probability that stopping occurs at length $t$ is then
\[
P_\theta(\tau^{(n)}_i=t)
=
\rho_{i,n,t}
\prod_{k<t}(1-\rho_{i,n,k}).
\]
The gate is trained to place probability mass on stopping times
inside $\mathcal{V}^{(n)}_i$,
which gives
\[
\mathcal{L}^{(i)}_{\mathrm{stop}}
=
-
\frac{1}{N}
\sum_{n=1}^{N}
\log
\sum_{t\in \mathcal{V}^{(n)}_i}
P_\theta(\tau^{(n)}_i=t).
\]

\subsection{Surrogate Latent Policy Optimization}
\label{sec:slpo}

As illustrated in Fig.~\ref{fig:slpo-overview},
SLPO couples latent-transition, answer-token, and stopping-time likelihoods
in a single reward-weighted objective.

\begin{figure*}[t]
    \centering
    \includegraphics[width=\textwidth]{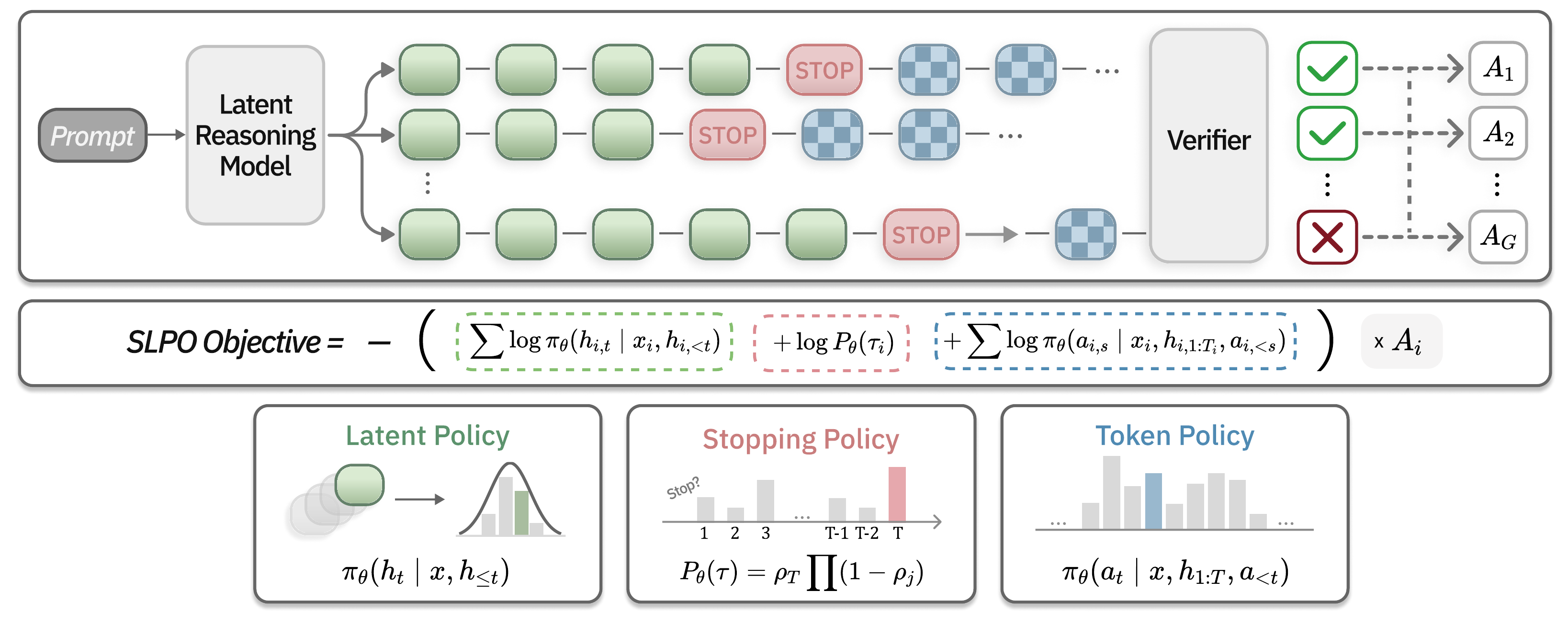}
    \vspace{-20pt}
    \caption{Overview of SLPO.
    Verifiable outcomes induce rollout advantages that weight the latent-surrogate, answer, and gate terms.}
    \label{fig:slpo-overview}
    \vspace{-8pt}
\end{figure*}

\paragraph{Surrogate transition likelihood.}
SLPO constructs a tractable Gaussian surrogate for stochastic latent transitions
from repeated MC-dropout evaluations.
For the prefix $(x_i,h_{i,<t})$,
we run $K$ stochastic forward evaluations with independent dropout masks during training:
\[
z_{i,t}^{(k)}
=
f_\theta(x_i,h_{i,<t};m_{i,t}^{(k)}),
\qquad
k=1,\ldots,K.
\]
These samples parameterize the Gaussian surrogate through
an empirical mean and a floored isotropic variance:
\[
\mu_{i,t}
=
\frac{1}{K}
\sum_{k=1}^{K}z_{i,t}^{(k)},
\qquad
\sigma_{i,t}^{2}
=
\frac{1}{Kd}
\sum_{k=1}^{K}
\bigl\|z_{i,t}^{(k)}-\mu_{i,t}\bigr\|_{2}^{2}
+\epsilon.
\]
The realized state $h_{i,t}$ is then scored under the isotropic Gaussian surrogate:
\[
\log \widetilde{\pi}_{\theta}
(h_{i,t}\mid x_i,h_{i,<t})
=
-\frac{d}{2}
\log\!\bigl(2\pi\sigma_{i,t}^{2}\bigr)
-
\frac{\|h_{i,t}-\mu_{i,t}\|_{2}^{2}}{2\sigma_{i,t}^{2}}.
\]
MC-dropout generates the sampled trajectories,
and the isotropic Gaussian supplies the per-step surrogate score used for advantage-weighted credit assignment
(App.~\ref{app:surrogate-objective}).
The sampled rollout states are treated as stop-gradient targets;
gradients flow through the recomputed moments $(\mu_{i,t},\sigma_{i,t}^{2})$.
Under fixed variance,
a positive advantage raises $\widetilde{\pi}_\theta(h_{i,t})$ by moving $\mu_{i,t}$ toward $h_{i,t}$,
while a negative advantage pushes $\mu_{i,t}$ away
(Prop.~\ref{prop:aw-matching}).

\paragraph{Rollout policy objective.}
For each input $x_i$,
we sample one or more latent rollouts,
decode final answers $a_i$,
and compute a verifiable reward
\[
R_i
=
R(a_i,a_i^\star).
\]
SLPO targets the expected rollout reward
\[
J(\theta)
=
\mathbb{E}_{(x_i,a_i^\star)\sim\mathcal{D}}
\mathbb{E}_{\xi_i\sim q_\theta(\cdot\mid x_i)}
\bigl[R(a_i,a_i^\star)\bigr].
\]
Here $q_\theta$ denotes the behavior trajectory law induced by MC-dropout,
answer decoding,
and the stopping gate.
In practice,
MC-dropout supplies behavior trajectories,
while the Gaussian surrogate converts their sampled latent states into a differentiable score.
The resulting update is the gradient of an explicitly defined empirical surrogate objective
(Prop.~\ref{prop:surr-pg});
App.~\ref{app:surrogate-objective} characterizes when this update follows the expected-reward gradient
and gives a sufficient condition for local improvement of $J(\theta)$.
After the stopping-gate cold start,
each rollout has an adaptive stopping time $\tau_i\le T_{\max}$.
Write $\xi_i=(h_{i,1:\tau_i},a_i)$.
The surrogate rollout score combines the latent-transition surrogate,
the answer-token likelihood,
and the stopping-time likelihood term from Sec.~\ref{sec:stop-gate-cold-start}:
\[
\begin{aligned}
\widetilde{\ell}_{\theta}(\xi_i\mid x_i)
&=
\sum_{t=1}^{\tau_i}
\log \widetilde{\pi}_{\theta}
(h_{i,t}\mid x_i,h_{i,<t})
+
\sum_{s=1}^{|a_i|}
\log \pi_\theta
(a_{i,s}\mid x_i,h_{i,1:\tau_i},a_{i,<s})
\\
&\qquad
+
\log P_\theta(\tau_i).
\end{aligned}
\]
The gate term assigns credit to stopping decisions:
higher-reward stopping times reinforce continue/stop probabilities through $\rho_{i,t}$.
With advantage $A_i=R_i-b_i$,
where $b_i$ may be a group-relative or leave-one-out baseline,
the surrogate loss is
\[
\mathcal{L}_{\mathrm{RL}}
=
-
A_i
\widetilde{\ell}_{\theta}(\xi_i\mid x_i).
\]
Differentiating this loss yields the detached empirical update characterized in App.~\ref{app:surrogate-objective}
(Prop.~\ref{prop:surr-pg});
RLOO or GRPO specify the stop-gradient advantage $A_i$.
The composite score thereby extends outcome-level credit into the latent recurrence
while jointly optimizing answer generation and trajectory length.

\subsection{Inference}
\label{sec:inference}
At inference,
latent reasoning proceeds sequentially.
After each latent state,
the stopping head emits a stop probability,
and the realized stopping time is the first step where this probability exceeds a threshold.
The final answer is decoded only after this stopping event.
Thus adaptive computation changes the number of latent steps used per instance
without requiring inference-time enumeration of candidate prefixes.
We leave the algorithm details to App.~\ref{app:stop-gate-inference}.

\section{Experiment Setup}
\label{sec:exp-setup}

Our experiments test whether SLPO (i) consistently improves outcome-level performance across latent reasoners,
(ii) remains compatible with different policy-optimization algorithms,
(iii) transfers across continuous and soft latent interfaces,
and (iv) converts a fixed latent budget into difficulty-adaptive test-time computation.

\paragraph{Datasets.}
We evaluate SLPO on grade-school mathematical reasoning.
We train on \texttt{GSM8K-Aug}~\citep{coconut},
and evaluate on three held-out benchmarks:
(1)~\texttt{GSM8K-Test}, the official GSM8K test split;
(2)~\texttt{GSM-Hard}~\citep{parallel-latent-tts},
a harder variant of GSM8K-Test with larger-magnitude numbers;
and (3)~\texttt{MultiArith}~\citep{parallel-latent-tts},
a multi-step arithmetic set.
The soft-latent generalization experiment additionally evaluates \texttt{MATH500},
\texttt{AIME2025}, and \texttt{AMC23}, following prior works~\citep{coconut,codi,colar,parallel-latent-tts}.

\paragraph{Models and baselines.}
We apply SLPO to publicly released COCONUT~\citep{coconut}
and CODI~\citep{codi} checkpoints,
using GPT-2 (124M) and Llama-3.2-1B-Instruct backbones.
%
We compare our variants against commonly used latent-reasoning baselines:
\textbf{COCONUT}~\citep{coconut} and \textbf{CODI}~\citep{codi},
\textbf{CoLaR}~\citep{colar}, \textbf{ReGuLaR}~\citep{regular}, \textbf{DART}~\citep{dart}, and \textbf{Latent-SFT}~\citep{latent-sft},
and the explicit-reasoning methods \textbf{CoT-SFT} and \textbf{iCoT}~\citep{coconut}.
Detailed baseline configurations are provided in App.~\ref{app:baselines}.
%
For soft-latent generalization experiments (Sec.~\ref{sec:soft-latent}),
we apply SLPO to Llama-3.2-1B-Instruct and Llama-3.2-3B-Instruct,
and compare against CoT with soft-token inference~\citep{soft-thinking}
and LEPO~\citep{lepo};
the corresponding setup appears in App.~\ref{app:soft-latent}.

\paragraph{Training and Evaluation Protocols.}
Unless otherwise specified,
\texttt{+SLPO} first trains the stopping gate and then optimizes SLPO with RLOO on \texttt{GSM8K-Aug}
using $(K,G)=(4,8)$.
We report \textbf{Acc} from a single deterministic rollout (dropout disabled), with
\textbf{Pass@$k$} for $k \in \{8,16\}$ draws $k$ independent MC-dropout rollouts ($p=0.1$) per problem,
we repeat this evaluation with three independent dropout seeds and report the mean.
(see further details in App.~\ref{app:inference}).
For Tab.~\ref{tab:main-results},
we additionally report average reasoning length (\#L):
token-chain length for explicit CoT,
the fixed latent budget for ungated latent baselines,
and the realized stopping time under $T_{\max}=12$ for \texttt{+SLPO}.

\section{Main Results}

\subsection{Consistent Scaling Across Latent Reasoners}
\label{sec:results-latent-backbones}

\begin{table}[t]
    \centering
    \caption{Comparison between latent-reasoning baselines and our variants.
    \textbf{Acc} is deterministic accuracy with dropout disabled;
    \textbf{Pass@$k$} for $k \in \{8,16\}$ is the mean over three independent MC-dropout evaluation seeds ($p=0.1$).}
    \label{tab:slpo-main-results}
    \begingroup
    \tablesetup
    \fontsize{9.2}{9.7}\selectfont
    \resizebox{\linewidth}{!}{%
    \begin{tabular}{l*{12}{S[table-format=2.2]}}
        \toprule
        \multirow{2}{*}{Method}
        & \multicolumn{3}{c}{GSM8K}
        & \multicolumn{3}{c}{GSM-Hard}
        & \multicolumn{3}{c}{MultiArith}
        & \multicolumn{3}{c}{Average} \\
        \cmidrule(lr){2-4}
        \cmidrule(lr){5-7}
        \cmidrule(lr){8-10}
        \cmidrule(lr){11-13}
        & \multicolumn{1}{c}{Acc} & \multicolumn{1}{c}{P@8} & \multicolumn{1}{c}{P@16} & \multicolumn{1}{c}{Acc} & \multicolumn{1}{c}{P@8} & \multicolumn{1}{c}{P@16} & \multicolumn{1}{c}{Acc} & \multicolumn{1}{c}{P@8} & \multicolumn{1}{c}{P@16} & \multicolumn{1}{c}{Acc} & \multicolumn{1}{c}{P@8} & \multicolumn{1}{c}{P@16} \\
        \midrule
        \rowcolor{TableGroup}\multicolumn{13}{@{}l}{\textbf{\texttt{GPT2--124M}}} \\
        \midrule
        COCONUT
        & 34.12 & 45.79 & 49.13
        & 7.66 & 10.70 & 12.06
        & 80.86 & 88.79 & 91.38
        & 40.88 & 48.43 & 50.86 \\
        \methodtag{COCONUT+SLPO}
        & \bfseries 35.63 & \bfseries 49.13 & \bfseries 51.55
        & \bfseries 7.66 & \bfseries 10.85 & \bfseries 12.52
        & \bfseries 83.10 & \bfseries 91.38 & \bfseries 92.59
        & \bfseries 42.13 & \bfseries 50.45 & \bfseries 52.22 \\
        CODI
        & 42.30 & 52.08 & 54.36
        & 9.26 & 11.84 & 12.67
        & 90.21 & 96.90 & 97.41
        & 47.59 & 53.61 & 54.81 \\
        \methodtag{CODI+SLPO}
        & \bfseries 42.76 & \bfseries 54.13 & \bfseries 56.71
        & \bfseries 9.71 & \bfseries 12.06 & \bfseries 13.05
        & \bfseries 90.52 & \bfseries 97.24 & \bfseries 97.76
        & \bfseries 47.66 & \bfseries 54.48 & \bfseries 55.84 \\
        \midrule
        \rowcolor{TableGroup}\multicolumn{13}{@{}l}{\textbf{\texttt{Llama3.2--1B}}} \\
        \midrule
        COCONUT
        & 21.23 & 24.64 & 38.13
        & 4.55 & 6.98 & 9.94
        & 41.1 & 45.00 & 63.10
        & 22.29 & 25.54 & 37.06 \\
        \methodtag{COCONUT+SLPO}
        & \bfseries 22.37 & \bfseries 30.48 & \bfseries 41.85
        & \bfseries 5.77 & \bfseries 8.19 & \bfseries 10.93
        & \bfseries 46.90 & \bfseries 57.07 & \bfseries 71.38
        & \bfseries 25.01 & \bfseries 31.91 & \bfseries 41.38 \\
        CODI
        & 55.22 & 63.91 & 67.48
        & 12.82 & 15.02 & 15.63
        & 95.52 & 98.28 & 98.79
        & 54.59 & 59.07 & 60.63 \\
        \methodtag{CODI+SLPO}
        & \bfseries 55.27 & \bfseries 65.13 & \bfseries 70.28
        & \bfseries 13.20 & \bfseries 15.86 & \bfseries 16.77
        & \bfseries 96.38 & \bfseries 98.79 & \bfseries 99.48
        & \bfseries 54.95 & \bfseries 59.93 & \bfseries 62.18 \\
        \bottomrule
    \end{tabular}
    }
    \endgroup
\end{table}

Tab.~\ref{tab:slpo-main-results} shows a consistent latent test-time scaling effect:
SLPO raises both Pass@8 and Pass@16 on all 12 backbone--dataset combinations.
The gains span curriculum-trained COCONUT and distillation-based CODI,
GPT-2 and Llama-3.2-1B,
and all three held-out benchmarks without altering the underlying vector recurrence.
Deterministic accuracy also improves on average for every backbone--reasoner pair.
The improvement reaches $12.07$ percentage points,
or $26.8\%$ relative,
for Pass@8 on MultiArith with Llama-3.2-1B COCONUT.
This substantial expansion in successful parallel rollouts is the clearest signature of the latent policy elicited by outcome-reward optimization.

\subsection{Accuracy with Adaptive Latent Computation}
\label{sec:results-broader}

\begin{table}[t]
    \centering
    \caption{Broader model results on grade-school mathematical reasoning benchmarks.
    \textbf{Acc} is deterministic accuracy.
    \#L is average reasoning length:
    fixed latent budget for ungated COCONUT/CODI ($T_{\max}=6$),
    realized stopping time for \texttt{+SLPO} under $T_{\max}=12$,
    and the corresponding length metric reported by each baseline.
    DART and Latent-SFT Acc/\#L are cited from~\citet{dart,latent-sft}.}
    \label{tab:main-results}
    \begingroup
    \tablesetup
    \fontsize{10.8}{11.3}\selectfont
    \sisetup{round-mode=places,round-precision=2,round-pad=true}
    \resizebox{\linewidth}{!}{%
    \begin{tabular}{>{\columncolor{white}}w{l}{1.8cm} w{l}{3.8cm} l*{8}{S[table-format=2.4]}}
        \toprule
        \multirow{2}{*}{Type}
        & \multirow{2}{*}{Method}
        & \multicolumn{2}{c}{GSM8K}
        & \multicolumn{2}{c}{GSM-Hard}
        & \multicolumn{2}{c}{MultiArith}
        & \multicolumn{2}{c}{Average} \\
        \cmidrule(lr){3-4}
        \cmidrule(lr){5-6}
        \cmidrule(lr){7-8}
        \cmidrule(lr){9-10}
        & & \multicolumn{1}{c}{Acc} & \multicolumn{1}{c}{\#L}
        & \multicolumn{1}{c}{Acc} & \multicolumn{1}{c}{\#L}
        & \multicolumn{1}{c}{Acc} & \multicolumn{1}{c}{\#L}
        & \multicolumn{1}{c}{Acc} & \multicolumn{1}{c}{\#L} \\
        \midrule
        \rowcolor{TableGroup}\multicolumn{10}{@{}l}{\textbf{\texttt{GPT2--124M}}} \\
        \midrule
        \textit{Explicit}
        & iCoT
        & 30.12 & 0
        & 5.73 & 0
        & 55.52 & 0
        & 30.41 & 0 \\
        & CoT-SFT
        & 44.13 & 25.1
        & 9.83 & 28.92
        & 90.72 & 10.92
        & 48.2 & 21.6 \\
        \midrule
        \textit{Latent}
        & COCONUT
        & 34.14 & 6
        & 7.66 & 6
        & 80.86 & 6
        & 40.87 & 6 \\
        & CODI
        & 42.30 & 6
        & 9.26 & 6
        & 90.21 & 6
        & 47.59 & 6 \\
        & \methodtag{COCONUT+SLPO}
        & 35.63 & 5.73
        & 7.66 & 6.08
        & 83.10 & 6.28
        & 42.13 & 6.03 \\
        & \methodtag{CODI+SLPO}
        & \bfseries 42.76 & 11.83
        & \bfseries 9.71 & 11.94
        & \bfseries 90.52 & 11.44
        & \bfseries 47.66 & 11.74 \\
        \midrule
        \rowcolor{TableGroup}\multicolumn{10}{@{}l}{\textbf{\texttt{Llama3.2--1B}}} \\
        \midrule
        \textit{Explicit}
        & iCoT
        & 19.84 & 0
        & 3.30 & 0
        & 38.18 & 0
        & 20.62 & 0 \\
        & CoT-SFT
        & 54.10 & 25.42
        & 15.66 & 34.24
        & 99.32 & 13.71
        & 56.31 & 24.46 \\
        \midrule
        \textit{Latent}
        & COCONUT
        & 21.23 & 6
        & 4.55 & 6
        & 41.12 & 6
        & 22.29 & 6 \\
        & CODI
        & 55.22 & 6
        & 12.82 & 6
        & 95.52 & 6
        & 54.59 & 6 \\
        & CoLaR
        & 26.57 & 5.56
        & 5.87 & 7.01
        & 86.39 & 3.23
        & 39.6 & 5.27 \\
        & ReGuLaR
        & 34.86 & 3.69
        & 8.27 & 4.12
        & 89.20 & 2.28
        & 44.13 & 3.36 \\
        & DART
        & 42.60 & 20.00
        & 10.90 & 20.00
        & 84.80 & 20.00
        & 46.10 & 20.00 \\
        & Latent-SFT
        & 52.41 & 12.82
        & 12.60 & 15.04
        & 96.82 & 7.19
        & 53.88 & 11.66 \\
        & \methodtag{COCONUT+SLPO}
        & 22.37 & 7.52
        & 5.77 & 7.88
        & 46.90 & 7.49
        & 25.01 & 7.63 \\
        & \methodtag{CODI+SLPO}
        & \bfseries 55.27 & 6.30
        & \bfseries 13.20 & 6.41
        & \bfseries 96.38 & 4.65
        & \bfseries 54.95 & 5.79 \\
        \bottomrule
    \end{tabular}
    }
    \endgroup
\end{table}

Tab.~\ref{tab:main-results} widens the comparison to explicit CoT-SFT, iCoT, and latent baselines.
On Llama-3.2-1B,
CODI+SLPO achieves the strongest average accuracy among latent methods with $5.79$ latent steps on average;
compared with Latent-SFT,
it raises average accuracy by $1.05$ percentage points while using $50.3\%$ fewer latent steps.
On GSM8K,
CODI+SLPO surpasses CoT-SFT by $1.17$ percentage points while using $75.2\%$ fewer reasoning steps.
The same optimization improves COCONUT,
showing that SLPO strengthens both curriculum-trained and distillation-trained latent initializations.
On GPT-2,
CODI+SLPO likewise achieves the strongest average accuracy among the latent methods.
These results position SLPO as an outcome-optimization layer that improves latent reasoning while allowing the learned gate to allocate,
rather than uniformly spend,
the available compute budget.

\section{Analysis}

We analyze SLPO from four perspectives: how well it generalizes across policy-optimization algorithms, how well it extends to soft latent inference, the geometry of latent rollouts it produces, and how the stopping gate adapts the allocation of latent computation to problem difficulty.
See further analyses of rollout hyperparameters and computational cost in
App.~\ref{app:rollout-hyperparam} and App.~\ref{app:computational-cost}.

\subsection{Generalization Across Policy-Optimization Algorithms}
\label{sec:rl-optimizer}

SLPO defines a surrogate likelihood on hidden transitions
that converts verifiable final-answer rewards into vector-space credit,
as formalized in Sec.~\ref{sec:slpo}.
Here we analyze whether SLPO adapts across outcome-reward policy-optimization algorithms
by comparing RLOO and GRPO under the same surrogate, rollout budget, and latent backbones.

\begin{figure}[t]
    \centering
    \includegraphics[width=\linewidth]{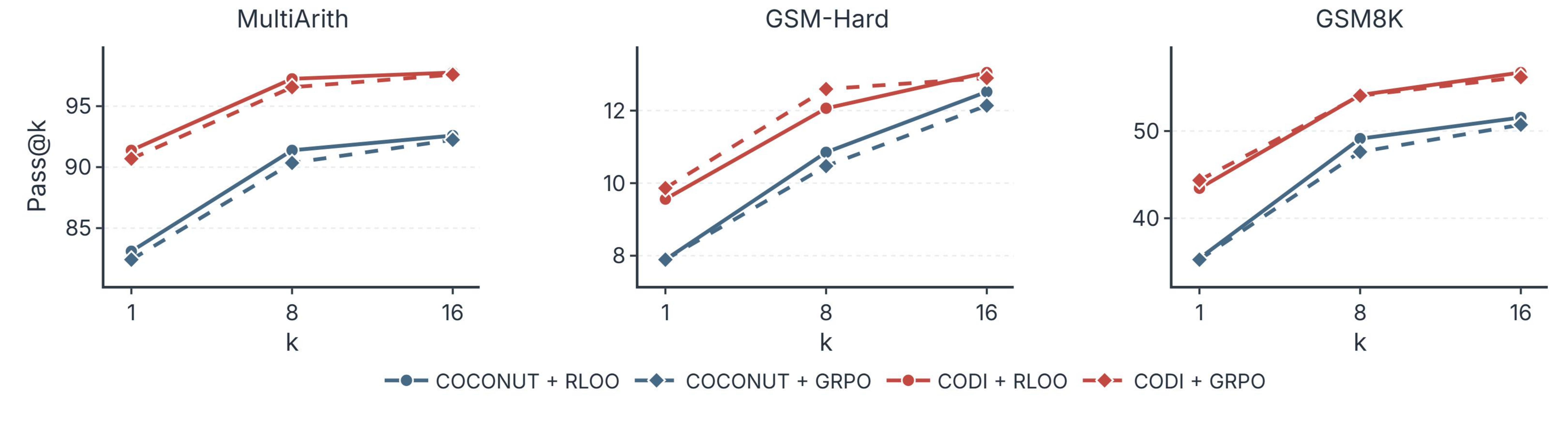}
    \caption{Pass@$k$ for $k \in \{1,8,16\}$ under SLPO with RLOO versus GRPO on COCONUT and CODI.
    Blue curves denote COCONUT; red curves denote CODI.
    Solid curves use RLOO and dashed curves use GRPO.
    Full values appear in App.~\ref{app:rl-optimizer}.}
    \label{fig:rl-optimizer-passk}
\end{figure}

Fig.~\ref{fig:rl-optimizer-passk} shows closely aligned scaling curves for RLOO and GRPO
across COCONUT and CODI on MultiArith, GSM-Hard, and GSM8K.
Both algorithms preserve SLPO's gains over the released backbone initialization across latent reasoners and sampling budgets.
The surrogate transition likelihood therefore provides a reusable outcome-reward interface
across distinct advantage estimators.

SLPO targets outcome-reward optimization over complete latent trajectories
in any autoregressive latent reasoner whose intermediate steps are vector states propagated by the backbone rather than vocabulary-level actions,
as formalized in Sec.~\ref{sec:slpo}.
Soft latent inference constitutes another vector-based latent paradigm~\citep{latent-grpo,lepo},
recurring on a vocabulary-weighted token embedding rather than the hidden state.
We thus evaluate SLPO under this setting to test whether the objective generalizes across vector-based latent inference mechanisms.
Tab.~\ref{tab:soft-latent-slpo-merged} compares SLPO with vocabulary-routed baselines on \texttt{Llama3.2--1B} and \texttt{Llama3.2--3B};
Fig.~\ref{fig:soft-seq-len-scaling} tracks mean rollout length during soft-token training on the 1B setup.
Implementation details appear in App.~\ref{app:soft-latent}.
\subsection{Generalization to Soft Latent Inference}
\label{sec:soft-latent}

\begin{table}[t]
    \centering
    \caption{Soft latent inference results on \texttt{Llama3.2--1B} and \texttt{Llama3.2--3B}.}
    \label{tab:soft-latent-slpo-merged}
    \begingroup
    \tablesetup
    \fontsize{9.2}{9.9}\selectfont
    \resizebox{\linewidth}{!}{%
    \begin{tabular}{l *{6}{S[table-format=3.2]} *{4}{S[table-format=2.2]}}
        \toprule
        \multirow{3}{*}{Method}
        & \multicolumn{6}{>{\columncolor{TableGroup}}c}{\textbf{\texttt{Llama3.2--1B}}}
        & \multicolumn{4}{>{\columncolor{TableGroup}}c}{\textbf{\texttt{Llama3.2--3B}}} \\
        \cmidrule(lr){2-7} \cmidrule(lr){8-11}
        & \multicolumn{3}{c}{GSM8K}
        & \multicolumn{3}{c}{MATH500}
        & \multicolumn{2}{c}{AIME 2025}
        & \multicolumn{2}{c}{AMC23} \\
        \cmidrule(lr){2-4} \cmidrule(lr){5-7}
        \cmidrule(lr){8-9} \cmidrule(lr){10-11}
        & {P@1} & {P@32} & {Len.}
        & {P@1} & {P@32} & {Len.}
        & {P@1} & {P@32} & {P@1} & {P@32} \\
        \midrule
        CoT & 37.30 & 80.50 & 200.18 & 21.10 & 70.40 & 599.02
        & 0.52 & 13.33 & 21.48 & 75.00 \\
        CoT (w/ latent inference) & 35.78 & 79.23 & 188.27 & 20.80 & 69.50 & 605.27
        & 0.83 & 16.67 & 21.64 & 77.50 \\
        GRPO & 45.56 & \bfseries 87.57 & 238.50 & 25.80 & 69.40 & 602.98
        & 0.42 & 13.33 & 24.14 & 75.00 \\
        GRPO (w/ latent inference) & 46.25 & 80.74 & 245.49 & 26.60 & 69.00 & 581.73
        & 0.42 & 16.67 & 23.98 & 72.50 \\
        Soft Tokens & 44.96 & 81.24 & 175.45 & 22.40 & 68.74 & 575.02
        & 0.42 & 10.00 & 20.55 & 72.50 \\
        HRPO & 33.32 & 70.33 & 158.27 & 18.20 & 64.22 & 551.27
        & 0.42 & 6.67 & 17.50 & 67.50 \\
        LEPO & 40.56 & 77.56 & 197.43 & 24.80 & 70.20 & 572.27
        & 0.96 & 16.67 & 27.03 & 77.50 \\
        \methodtag{SLPO (ours)} & \bfseries 46.70 & 82.03 & \bfseries 256.65
        & \bfseries 27.20 & \bfseries 71.60 & \bfseries 642.32
        & \bfseries 3.33 & \bfseries 20.00 & \bfseries 32.50 & \bfseries 77.50 \\
        \bottomrule
    \end{tabular}%
    }
    \endgroup
\end{table}

Tab.~\ref{tab:soft-latent-slpo-merged} establishes that SLPO transfers beyond hidden-state recurrence.
At 1B,
SLPO achieves the best deterministic accuracy on both GSM8K and MATH500
and the best MATH500 Pass@32.
At 3B,
it attains the strongest results on the more selective AIME~2025 and AMC23 evaluations,
raising AIME~2025 Pass@1 from $0.96$ to $3.33$ ($3.47\times$)
and AMC23 Pass@1 from $27.03$ to $32.50$ ($+5.47$ percentage points; $20.2\%$ relative).
Fig.~\ref{fig:soft-seq-len-scaling} exposes the corresponding sequential scaling behavior:
SLPO steadily expands rollout length under reward optimization,
whereas GRPO and LEPO remain near initialization or contract.
SLPO thus elicits both stronger outcomes and increased test-time computation through a distinct soft-token latent interface.
\begin{wrapfigure}{r}{0.41\textwidth}
    \centering
    \includegraphics[width=\linewidth]{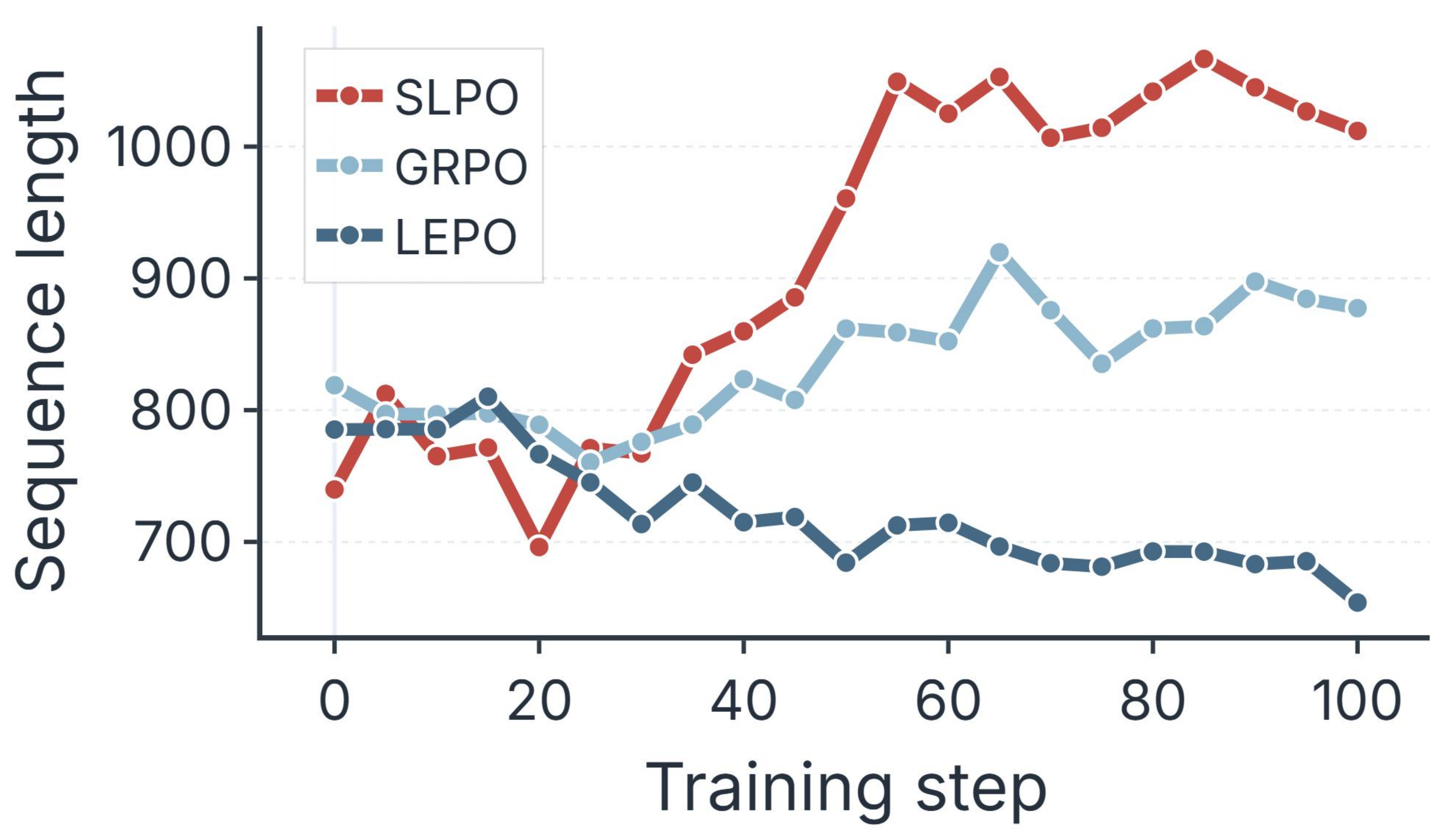}
    \caption{Mean generated sequence length during soft-token outcome-reward optimization on \texttt{Llama3.2--1B}.}
    \label{fig:soft-seq-len-scaling}
\end{wrapfigure}

\subsection{Latent Geometry Analysis}
\label{sec:latent-geometry}

Outcome-reward RL reinforces successful latent rollouts, but the 
underlying latent geometry change remains unclear.
We thus track two complementary signals to characterize how SLPO reshapes latent thinking geometry:
inter-step separation, the mean cosine distance between consecutive latent states $h_{i,t}$ and $h_{i,t+1}$ along each rollout,
and relative change in prefix effective rank.
Fig.~\ref{fig:latent-geometry-inter} reports the former and Fig.~\ref{fig:latent-geometry-rank} the latter;
see metric definitions, sampling details, and prefix-length rank curves in App.~\ref{app:latent-geometry}.

\begin{figure}[t]
    \centering
    \begin{subfigure}{0.4108\linewidth}
        \centering
        \includegraphics[width=\linewidth]{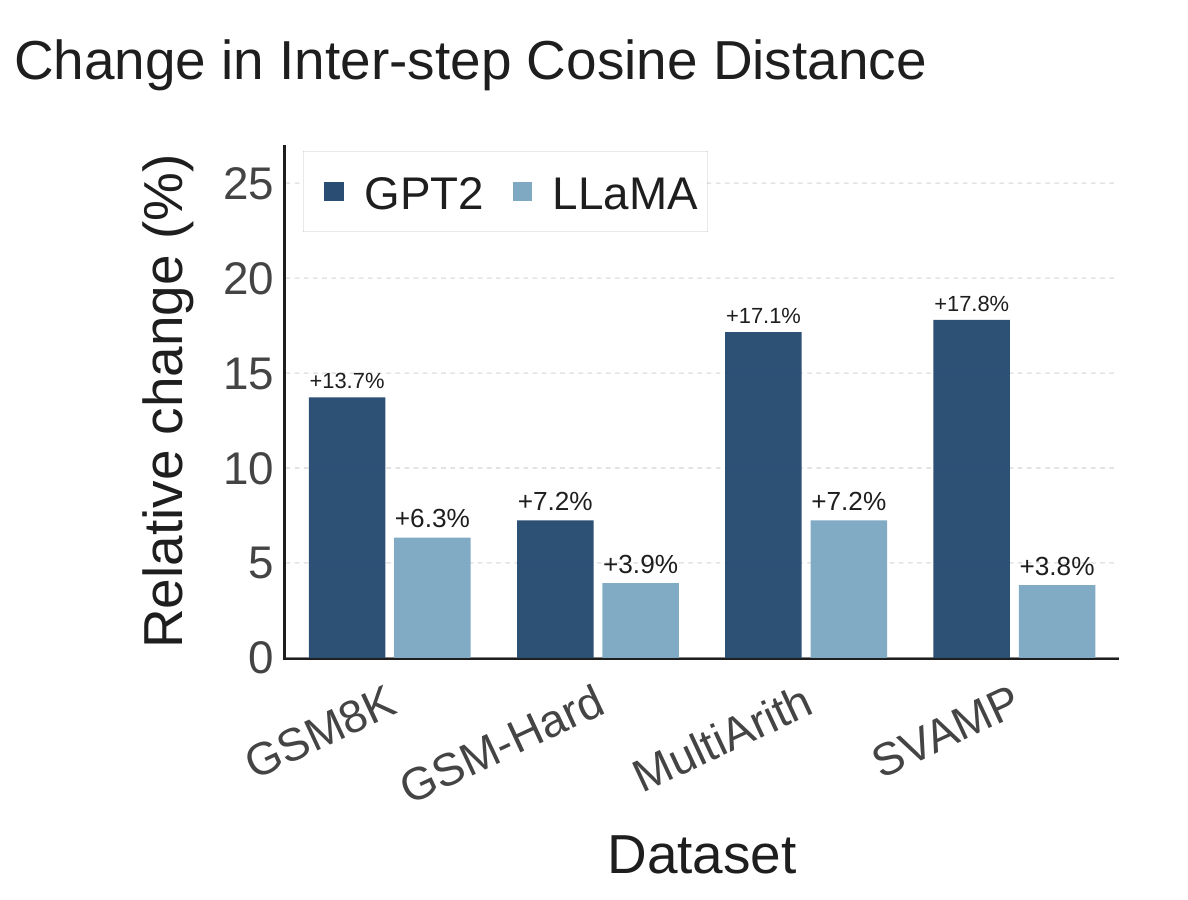}
        \caption{Inter-step cosine distance}
        \label{fig:latent-geometry-inter}
    \end{subfigure}
    \hfill
    \begin{subfigure}{0.5392\linewidth}
        \centering
        \includegraphics[width=\linewidth]{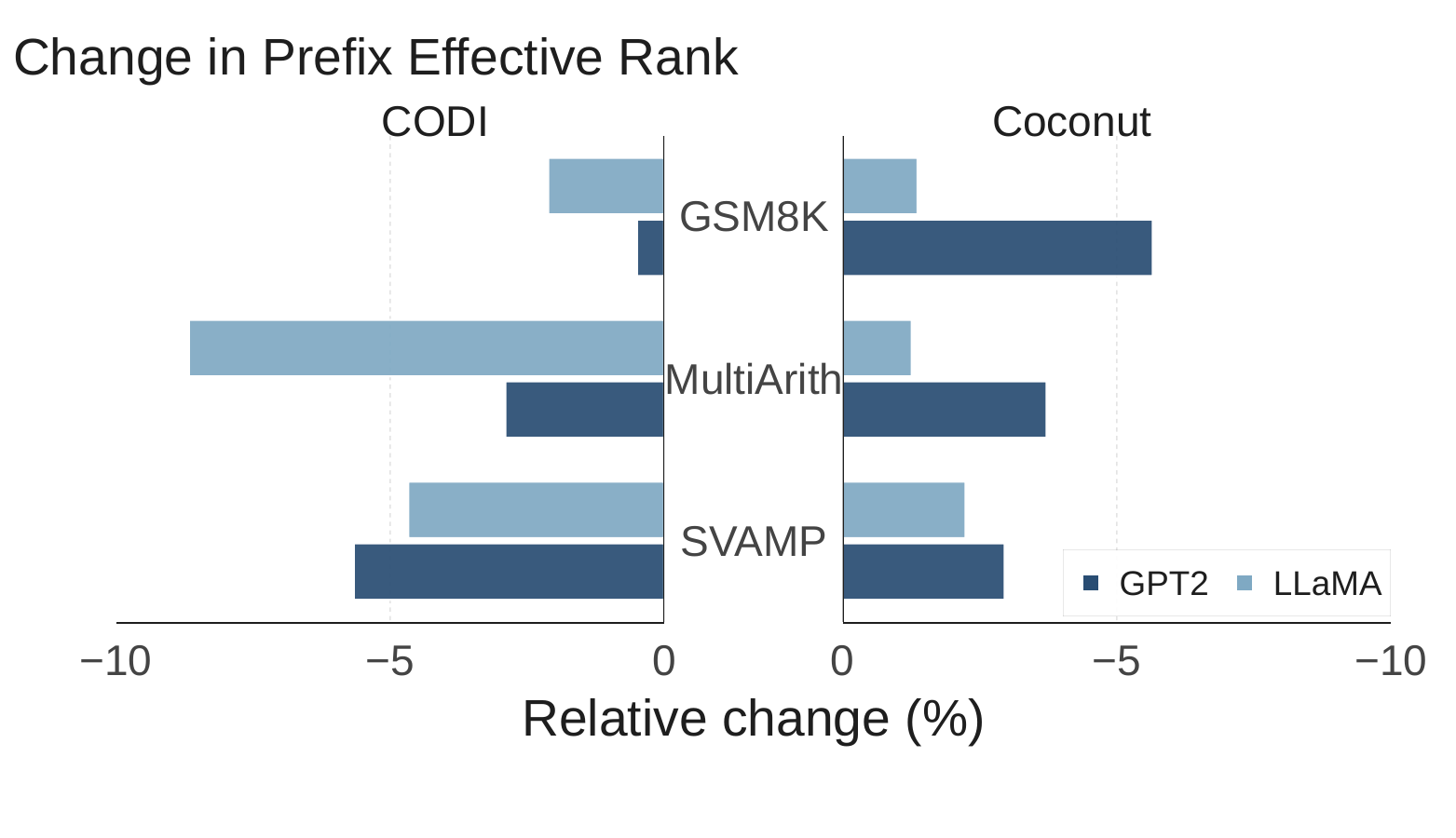}
        \caption{Relative rank change}
        \label{fig:latent-geometry-rank}
    \end{subfigure}
    \caption{Latent geometry pre- and post-SLPO.
    Fig.~\subref{fig:latent-geometry-inter}: relative change in mean cosine distance between consecutive latent states across backbones and datasets.
    Fig.~\subref{fig:latent-geometry-rank}: relative change in prefix effective rank across backbones, datasets;
    metric definitions and prefix-length rank curves appear in App.~\ref{app:latent-geometry}.}
    \label{fig:latent-geometry}
\end{figure}

For this geometry probe we additionally include \texttt{SVAMP}~\citep{coconut}, a grade-school set with lexical perturbations.
Fig.~\ref{fig:latent-geometry-inter} reports the relative change in mean inter-step cosine distance after SLPO across backbones and datasets.
Inter-step distance increases for every backbone--dataset pair,
with the largest change on GSM-Hard.
Successive latent states therefore become more differentiated after SLPO,
producing a clearer stage-wise progression through the latent thinking process.

Inter-step separation captures pairwise progression along each rollout.
We next track prefix effective rank,
which measures how many independent directions the latent prefix spans, with definition given in App.~\ref{app:latent-geometry-rank}.
Fig.~\ref{fig:latent-geometry-rank} reports the relative change in prefix effective rank after SLPO across backbones, datasets, and latent-training methods.
Effective rank decreases in every backbone--dataset--method combination.
The latent prefix therefore concentrates into a lower-dimensional subspace after SLPO.

Together, the rise in inter-step separation and this rank compression show that SLPO yields stronger differentiation between successive thinking steps while steering rollouts through a sharpened set of task-relevant directions.

\subsection{Difficulty-Adaptive Compute Allocation Analysis}

\begin{figure}[t]
    \centering
    \begin{subfigure}{0.46\linewidth}
        \centering
        \includegraphics[width=\linewidth]{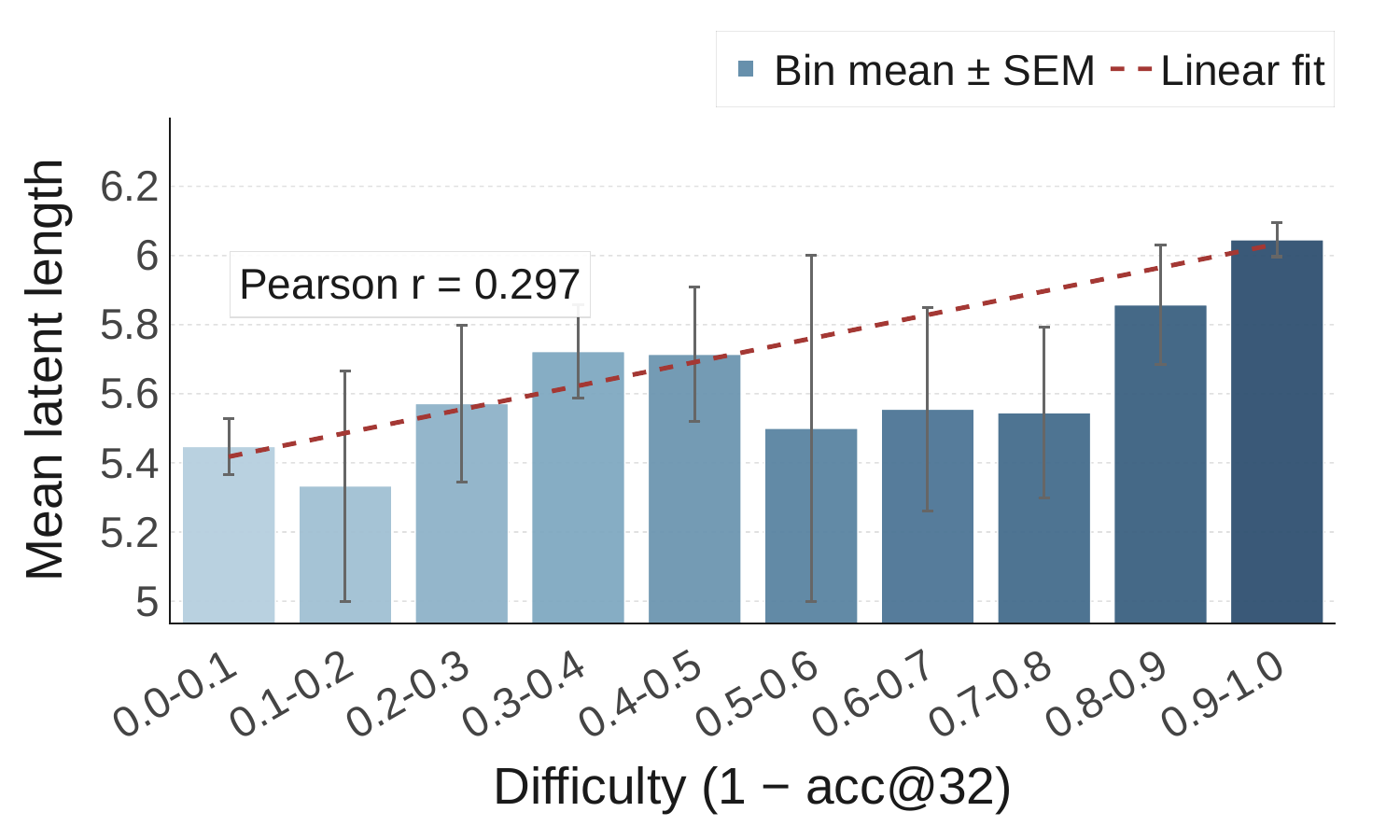}
        \caption{GSM8K validation}
        
    \end{subfigure}
    \hfill
    \begin{subfigure}{0.46\linewidth}
        \centering
        \includegraphics[width=\linewidth]{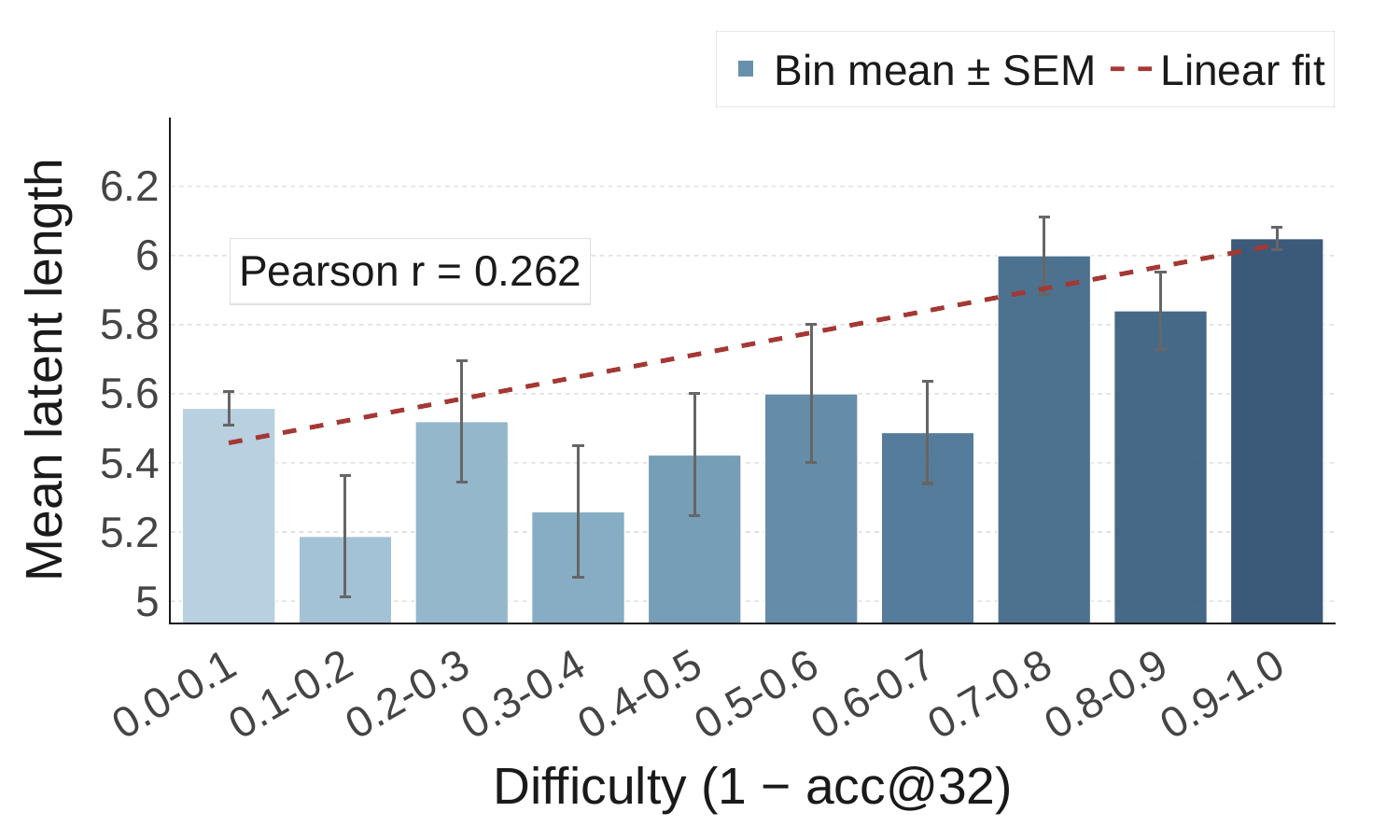}
        \caption{GSM8K test}
    \end{subfigure}
    \caption{Relationship between problem difficulty and latent length under the learned stopping gate.
    Bars show bin means with standard errors,
    and the dashed line shows the linear fit.
    Per-problem Pearson correlations are $r=0.30$ (validation) and $r=0.26$ (test),
    both with two-sided $p<0.001$.}
    \label{fig:stop-gate-difficulty}
\end{figure}

The stopping-gate cold start trains the stop gate from the correctness of answers decoded under varying latent lengths (Sec.~\ref{sec:stop-gate-cold-start}).
We test whether the resulting gate allocates latent computation by difficulty by measuring difficulty as the empirical failure rate $d_i = 1 - \mathrm{acc@32}_i$
and comparing gate-selected latent length across difficulty bins (Fig.~\ref{fig:stop-gate-difficulty}).

As shown in Fig.~\ref{fig:stop-gate-difficulty},
harder problems receive longer latent trajectories on both validation and test.
Easy bins concentrate on shorter prefixes,
while the hardest bins draw the longest hidden computation
(Pearson $r=0.30$ on validation and $r=0.26$ on test;
two-sided $t$-tests, both $p<0.001$).
The learned gate therefore provides adaptive test-time scaling in latent space:
harder instances draw more computation,
while easy instances terminate early rather than following a single global budget.

\section{Conclusion and Future Work}
\label{sec:conclusion}

We introduce SLPO for outcome-reward optimization in autoregressive latent reasoners.
SLPO defines a differentiable surrogate policy interface over MC-dropout latent transitions,
extending trajectory-level reward credit directly into vector-space reasoning.
A correctness-supervised stopping-gate cold start establishes a prior over stopping times,
allowing SLPO to optimize adaptive computation alongside latent reasoning.
Across continuous and soft thinking inference,
SLPO realizes latent test-time scaling through higher Pass@$k$ under parallel sampling
and longer latent trajectories on harder instances with improved deterministic accuracy.
Future work will extend SLPO to open-ended reasoning and multimodal latent architectures.


\bibliographystyle{unsrtnat}
\bibliography{references}


\appendix
\section{Related Work}

\paragraph{Latent reasoning.}
Latent reasoning instead autoregresses in hidden space,
compressing redundant explicit traces into continuous thoughts~\citep{efficient-reasoning-survey,latent-cot-survey,coconut,codi,colar,sim-cot,rot,regular,dart,latent-sft}.
These approaches compress reasoning through curriculum training, self-distillation,
dynamic or variational latent traces, and vocabulary-space representations~\citep{coconut,codi,colar,regular,dart,latent-sft,sim-cot,rot}.
These methods optimize latent states under SFT-style supervision from explicit reasoning,
typically with a fixed inference budget; parallel test-time scaling instead keeps weights frozen,
stochastically rolls out latent trajectories, and aggregates them with a learned reward model~\citep{parallel-latent-tts}.

\paragraph{Reinforcement learning.}
Outcome-reward RL has become central to scaling explicit CoT~\citep{rl-lrm-survey,deepseek-r1,rlvr-incentivizes}.
Token-policy objectives such as GRPO and RLOO remain standard~\citep{deepseekmath-grpo,effective-rl-reasoning},
but continuous hidden transitions expose no action log-probabilities~\citep{rlvr-limit,reasoning-boundary-paradox}.
Recent methods address this gap with soft-token rollouts, vocabulary-space latent tokens, or learned latent heads~\citep{lepo,latent-grpo,colar}.
These designs depend on token distributions, embedding mixtures, or architecture-specific latent heads,
so they do not carry over to arbitrary vector-based latent reasoners with plain hidden-state recurrence.
Outcome-reward optimization in this setting remains unexplored.

\section{Implementation Details}
\label{app:implementation}

This appendix supplements Sec.~\ref{sec:exp-setup} with model, training, and evaluation details omitted from the main text for space.

\subsection{Base Models and Initialization}
\label{app:base-models}

All SLPO experiments start from publicly available latent-reasoning checkpoints,
using official baseline releases when available and open third-party reimplementations otherwise.
COCONUT and CODI use GPT-2 (124M) or Llama-3.2-1B-Instruct,
depending on the row in the main tables;
CoLaR, ReGuLaR, DART, and Latent-SFT use Llama-3.2-1B-Instruct.
We do not retrain the imitation stage of these latent reasoners:
The stopping-gate cold start and SLPO are applied directly on top of these initialized weights.

\subsection{Baseline Configurations}
\label{app:baselines}

We use the default inference configuration of each latent baseline whenever possible.
Ungated COCONUT and CODI use a fixed budget of six latent thoughts,
matching the original papers;
\texttt{+SLPO} uses the same backbones with maximum budget $T_{\max}=12$
and reports the realized stopping time as \#L
(App.~\ref{app:stop-gate-cold-start}--\ref{app:stop-gate-inference}).
CoLaR uses thinking speed $2\times$ with a maximum latent budget of 64 steps.
ReGuLaR follows the released evaluation protocol in~\citet{regular}.
DART Acc and \#L are taken from~\citet{dart},
with \#L equal to their fixed Silent Thought count $C{=}20$
(non-autoregressive; all ST tokens are processed in one forward pass).
Latent-SFT Acc and \#L are taken from the $r{=}2$ Llama-3.2-1B setting in~\citet{latent-sft}
(GSM8K/GSM-Hard/MultiArith; averages exclude SVAMP).
For fair comparison with explicit CoT baselines,
CoT-SFT and iCoT numbers follow the training protocol reported by~\citet{coconut,codi,colar}.

For CoT-SFT and iCoT,
\#L is measured on the generated reasoning chain before answer prediction.
For latent methods without the stopping gate,
\#L is the fixed latent budget used at inference.

\subsection{Training Procedure}

\subsubsection{Stopping-gate Cold Start}
\label{app:stop-gate-cold-start}

The stopping-gate cold start is applied first on the same GSM8K-Aug split.
For each training problem,
the released backbone samples $N=4$ stochastic latent trajectories with rollout dropout ($p=0.1$) up to $T_{\max}=12$,
the same maximum budget used at SLPO training and gated inference
(raised from the ungated COCONUT/CODI horizon of six).
Candidate stopping lengths lie in $[T_{\min},T_{\max}]$ with $T_{\min}=3$.
For every trajectory and candidate prefix length,
answer decoding runs with dropout disabled to stabilize the online supervision signal.
The stopping head is trained with the first-stop objective in Sec.~\ref{sec:stop-gate-cold-start}.
We do not add a compute-aware auxiliary term.

\paragraph{Optimization.}
This stage uses Adafactor at learning rate $10^{-4}$ for $15$ epochs.
For GPT-2 (COCONUT and CODI),
we use a per-device batch size of $32$ with $2$ gradient-accumulation steps;
for Llama-3.2-1B (COCONUT and CODI),
we use a per-device batch size of $8$ with $8$ accumulation steps.
The stop gate itself is regularized with dropout $p=0.1$ during this cold start.
We evaluate every $500$ optimizer steps,
restoring the checkpoint with the highest validation accuracy under early stopping with patience $20$.

\subsubsection{SLPO Training}
\label{app:slpo-training}

\paragraph{Reward and advantage.}
Each training instance $(x_i,a_i^\star)$ receives a binary outcome reward
$R_i=\mathbb{I}[\hat{a}_i=a_i^\star]$ after answer parsing.
Advantages are estimated with RLOO over $G=8$ sampled latent trajectories per problem.
The surrogate transition likelihood at each step uses $K=4$ independent MC-dropout forwards.

\paragraph{Stochastic rollouts.}
Latent trajectories are sampled with MC-dropout at rate $p=0.1$ during rollout generation.
The realized latent states are treated as stop-gradient targets;
gradients flow through the recomputed moments $(\mu_{i,t},\sigma_{i,t}^{2})$ in Sec.~\ref{sec:slpo},
not through the rollout sampling path itself.

\paragraph{Optimization.}
We optimize with Adafactor at learning rate $10^{-6}$.
Training runs for up to $40{,}000$ optimizer steps under a constant schedule with warmup over the first $2{,}000$ steps.
We do not apply PPO-style clipping or a KL penalty:
the SLPO RLOO objective uses $\beta=0$ and omits reference-policy log-probability terms.
Answer-token likelihoods and the first-stop gate term are included in the rollout objective as in Sec.~\ref{sec:slpo}.

\subsection{Runtime Environment}

\paragraph{Hardware.}
SLPO is launched with Accelerate over $4$ processes on $4$ GPUs (4$\times$RTX 5880 Ada).
Each process uses a per-device batch size of $16$;
with $4$ GPUs,
the effective prompt batch is $64$.

\subsection{Inference and Evaluation}
\label{app:inference}
\subsubsection{Stopping-gate Inference}
\label{app:stop-gate-inference}

Alg.~\ref{alg:inference} gives the inference procedure after the stopping-gate cold start and SLPO.
Latent reasoning and answer decoding both run with dropout disabled,
and the stopping gate uses deterministic thresholding.
Unless otherwise noted,
$T_{\max}=12$ matches the cold-start budget in App.~\ref{app:stop-gate-cold-start}.

\begin{algorithm}[h]
\caption{Adaptive latent inference with stopping gate}
\label{alg:inference}
\begin{algorithmic}[1]
\small
\Require Input $x$; latent reasoner $f_\theta$ with stopping head $g_\theta$;
         maximum budget $T_{\max}$; stop threshold $\tau$
\State Disable dropout for latent updates and answer decoding
\State Initialize latent prefix $\tilde{h}_{<1}\leftarrow\emptyset$
\For{$t=1$ to $T_{\max}$}
    \State $\tilde{h}_t \leftarrow f_\theta(x,\tilde{h}_{<t})$
    \State $\rho_t \leftarrow \sigma(g_\theta(\tilde{h}_t))$
    \If{$\rho_t \ge \tau$}
        \State \textbf{break}
    \EndIf
\EndFor
\State Let $T^\star \leftarrow t$ be the realized stopping step
\State Decode final answer $\hat{a} \sim \pi_\theta(\cdot \mid x,\tilde{h}_{1:T^\star})$
\State \Return $\hat{a}$
\end{algorithmic}
\end{algorithm}

We sweep candidate thresholds $\{0.5,0.6,0.7,0.8,0.9\}$ on the validation split
and select the threshold with the highest validation accuracy.
When a single default threshold is required,
we use $0.5$.

\subsubsection{Accuracy and Pass@$k$ Estimation}
\textbf{Acc} uses a single deterministic forward pass per problem:
latent reasoning and answer decoding both run with dropout disabled,
following the stopping-gate inference path in Sec.~\ref{sec:inference}.
For $k>1$,
one Pass@$k$ evaluation draws $k$ independent stochastic latent rollouts per problem with MC-dropout ($p=0.1$);
a problem counts as solved if at least one trajectory yields the correct parsed answer.
Because dropout seeds affect the sampled trajectories,
we repeat this evaluation $R{=}3$ times with independent seeds
and report the mean Pass@$k$ over the three runs in Tab.~\ref{tab:slpo-main-results}.

\section{Supplementary Experimental Material}

\subsection{Policy-Optimization Algorithm Comparison}
\label{app:rl-optimizer}

This appendix reports the full Pass@$k$ values underlying Fig.~\ref{fig:rl-optimizer-passk}.
We keep the SLPO surrogate, rollout budget, and latent backbones fixed,
and swap only the outcome-reward policy-optimization algorithm between RLOO and GRPO.

\begin{table}[h]
    \centering
    \caption{Pass@$k$ under SLPO with RLOO versus GRPO on COCONUT and CODI.}
    \label{tab:rl-optimizer-passk}
    \begingroup
    \setlength{\tabcolsep}{4pt}
    \renewcommand{\arraystretch}{1.08}
    \resizebox{.75\linewidth}{!}{%
    \begin{tabular}{ll *{3}{S[table-format=2.2]} *{3}{S[table-format=2.2]}}
        \toprule
        & & \multicolumn{3}{c}{COCONUT} & \multicolumn{3}{c}{CODI} \\
        \cmidrule(lr){3-5}
        \cmidrule(lr){6-8}
        Dataset & Algorithm & {Pass@1} & {Pass@8} & {Pass@16} & {Pass@1} & {Pass@8} & {Pass@16} \\
        \midrule
        \multirow{2}{*}{MultiArith} & RLOO & 83.10 & 91.38 & 92.59 & 90.52 & 97.24 & 97.76 \\
        & GRPO & 82.41 & 90.34 & 92.24 & 90.69 & 96.55 & 97.59 \\
        \addlinespace[0.2em]
        \multirow{2}{*}{GSM-Hard} & RLOO & 7.66 & 10.85 & 12.52 & 9.71 & 12.06 & 13.05 \\
        & GRPO & 7.89 & 10.47 & 12.14 & 9.86 & 12.59 & 12.90 \\
        \addlinespace[0.2em]
        \multirow{2}{*}{GSM8K} & RLOO & 35.63 & 49.13 & 51.55 & 42.76 & 54.13 & 56.71 \\
        & GRPO & 35.25 & 47.61 & 50.72 & 44.35 & 54.06 & 56.18 \\
        \bottomrule
    \end{tabular}
    }
    \endgroup
\end{table}

\subsection{Soft Latent Inference Transfer}
\label{app:soft-latent}

This appendix supplements Sec.~\ref{sec:soft-latent} with the setup behind Tab.~\ref{tab:soft-latent-slpo-merged}.
We keep the soft-token inference mechanism unchanged~\citep{soft-thinking},
apply SLPO on Llama3.2--1B and Llama3.2--3B,
and compare against CoT, vocabulary-routed GRPO, Soft Tokens, HRPO, and LEPO~\citep{lepo} under the same stack.
This transfer setting uses SLPO without the stopping-gate cold start.

\paragraph{Soft-token interface.}
Each latent step is a probability-weighted embedding rather than a backbone hidden state.
The model maps logits $\ell_t$ to $p_t=\mathrm{softmax}(\ell_t/\tau)$ and forms
\[
z_t = p_t^\top E.
\]
Answer generation after the latent block still uses the ordinary token likelihood.

\paragraph{SLPO objective.}
We apply the surrogate in Sec.~\ref{sec:slpo} by identifying $\tilde{h}_{i,t}\equiv z_{i,t}$ and scoring transitions in embedding space.
The isotropic Gaussian moments and likelihood follow Sec.~\ref{sec:slpo},
with variance floor $\epsilon=10^{-6}$.
On the optimized latent prefix,
the surrogate term $-\widehat{A}_i\log\widetilde{\pi}_\theta(z_{i,t}\mid x_i,z_{i,<t})$ replaces the soft-token cross-entropy;
subsequent answer tokens retain the standard token log-probability objective.

\paragraph{Training and evaluation.}
Rollouts use a latent block of length $32$.
We score only an early prefix with the surrogate---$8$ steps by default, with a $16$-step variant also evaluated---using $K=2$ MC-dropout samples per step.
Stochastic forwards follow App.~\ref{app:slpo-training} ($p=0.1$).
Outcome rewards and group-relative advantages follow the GRPO-style soft-token baseline.
We evaluate GSM8K and MATH500 at 1B scale,
and AIME~2025 and AMC23 at 3B scale (Tab.~\ref{tab:soft-latent-slpo-merged});
the 3B runs keep the same interface and surrogate settings,
changing only the backbone scale.
Other optimization details follow the soft-token baseline recipe shared with LEPO and vocabulary-routed GRPO.

\subsection{Latent Geometry Analysis}
\label{app:latent-geometry}

This appendix supplements Sec.~\ref{sec:latent-geometry} with the metric definitions, aggregation rules, and figure construction used in the latent-geometry analysis.

\subsubsection{Sampling Protocol}
\label{app:latent-geometry-sampling}

We compare pre- and post-SLPO checkpoints on the same evaluation instances.
For each problem $x_i$,
we draw $N$ stochastic latent trajectories with MC-dropout at rate $p=0.1$,
using the same rollout mechanism as SLPO training (Sec.~\ref{sec:slpo}).
Each trajectory is
\[
h_{i,1:T}^{(n)}=\bigl(h_{i,1}^{(n)},\ldots,h_{i,T}^{(n)}\bigr),
\qquad n=1,\ldots,N,
\]
with fixed latent budget $T=6$.
Unless stated otherwise,
we evaluate the first $32$ problems per dataset,
use $N=16$ dropout samples per problem,
and report metrics on COCONUT and CODI with GPT-2 or Llama backbones.

\subsubsection{Step Representations}
\label{app:latent-geometry-representations}

Dropout sampling induces local stochasticity around each latent step.
To obtain a representative latent vector for step $t$,
we aggregate sampled states across trajectories:
\begin{equation}
c_{i,t}
=
\frac{1}{N}\sum_{n=1}^{N} h_{i,t}^{(n)}.
\label{eq:step-representation}
\end{equation}
The sequence $(c_{i,1},\ldots,c_{i,T})$ is the step-wise representation used by both metrics below.

\subsubsection{Inter-Step Separation}
\label{app:latent-geometry-inter}

Inter-step separation measures how far successive step representations move apart along this trajectory.
We use cosine distance
\begin{equation}
\mathrm{dcos}(u,v)
=
1-\frac{u^\top v}{\|u\|_2\|v\|_2},
\label{eq:cosine-distance}
\end{equation}
and define the per-problem inter-step score as the mean over consecutive step representations:
\begin{equation}
D_i^{\mathrm{inter}}
=
\frac{1}{T-1}\sum_{t=1}^{T-1}
\mathrm{dcos}(c_{i,t},c_{i,t+1}).
\label{eq:inter-step-distance}
\end{equation}
Larger $D_i^{\mathrm{inter}}$ means successive latent steps occupy more separated positions in hidden space.
For a dataset $\mathcal{S}$,
we average over problems,
\[
\bar{D}^{\mathrm{inter}}
=
\frac{1}{|\mathcal{S}|}\sum_{i\in\mathcal{S}} D_i^{\mathrm{inter}},
\]
and report the relative change from pre- to post-SLPO,
\begin{equation}
\Delta^{\mathrm{inter}}
=
100\cdot
\frac{
\bar{D}^{\mathrm{inter}}_{\mathrm{post}}
-
\bar{D}^{\mathrm{inter}}_{\mathrm{pre}}
}{
\bar{D}^{\mathrm{inter}}_{\mathrm{pre}}
}.
\label{eq:inter-step-change}
\end{equation}
Fig.~\subref{fig:latent-geometry-inter} plots $\Delta^{\mathrm{inter}}$ for COCONUT across GSM8K, GSM-Hard, MultiArith, and SVAMP.

\subsubsection{Prefix Effective Rank}
\label{app:latent-geometry-rank}

To characterize how the latent prefix distributes across the hidden subspace as $t$ grows,
we track how many independent directions the latent prefix $\{c_1,\ldots,c_t\}$ spans as $t$ grows.
For prefix length $t\ge 2$,
stack the first $t$ step representations into a matrix
\[
C_{i,t}
=
\begin{bmatrix}
c_{i,1}^\top\\
\vdots\\
c_{i,t}^\top
\end{bmatrix}
\in\mathbb{R}^{t\times d},
\]
row-center across latent steps,
\[
\tilde{C}_{i,t}
=
C_{i,t}-\mathbf{1}\bar{c}_{i,t}^\top,
\qquad
\bar{c}_{i,t}
=
\frac{1}{t}\sum_{j=1}^{t}c_{i,j},
\]
and compute the effective rank
\begin{equation}
r_{\mathrm{eff}}(i,t)
=
\frac{\|\tilde{C}_{i,t}\|_F^2}{\|\tilde{C}_{i,t}\|_2^2}
=
\frac{\sum_{j=1}^{\min(t,d)}\sigma_j\!\bigl(\tilde{C}_{i,t}\bigr)^2}{\sigma_1\!\bigl(\tilde{C}_{i,t}\bigr)^2},
\label{eq:effective-rank}
\end{equation}
where $\sigma_j(\cdot)$ denotes the $j$-th singular value.
This is the stable rank of the centered step-representation matrix:
it equals $1$ when the step representations lie on one line,
and grows as the prefix spans more independent directions.
At each $t$,
we average $r_{\mathrm{eff}}(i,t)$ over problems and plot the resulting before/after curves.
The $x$-axis therefore marks prefix length rather than a single latent state;
because adding step representations also raises the maximum attainable rank,
the upward slope along $t$ partly reflects accumulation.
The comparison of interest is pre- versus post-SLPO at the same $t$.

\subsubsection{Figure Construction}
\label{app:latent-geometry-figures}

Fig.~\ref{fig:latent-geometry} in the main text has two panels.
Fig.~\subref{fig:latent-geometry-inter} reports $\Delta^{\mathrm{inter}}$ for COCONUT (\eqref{eq:inter-step-distance}--\eqref{eq:inter-step-change}).
Fig.~\subref{fig:latent-geometry-rank} reports $r_{\mathrm{eff}}(i,t)$ from \eqref{eq:effective-rank} on two representative CODI settings:
GPT-2 on SVAMP and Llama on MultiArith.

\begin{figure}[t]
    \centering
    \includegraphics[width=0.95\linewidth]{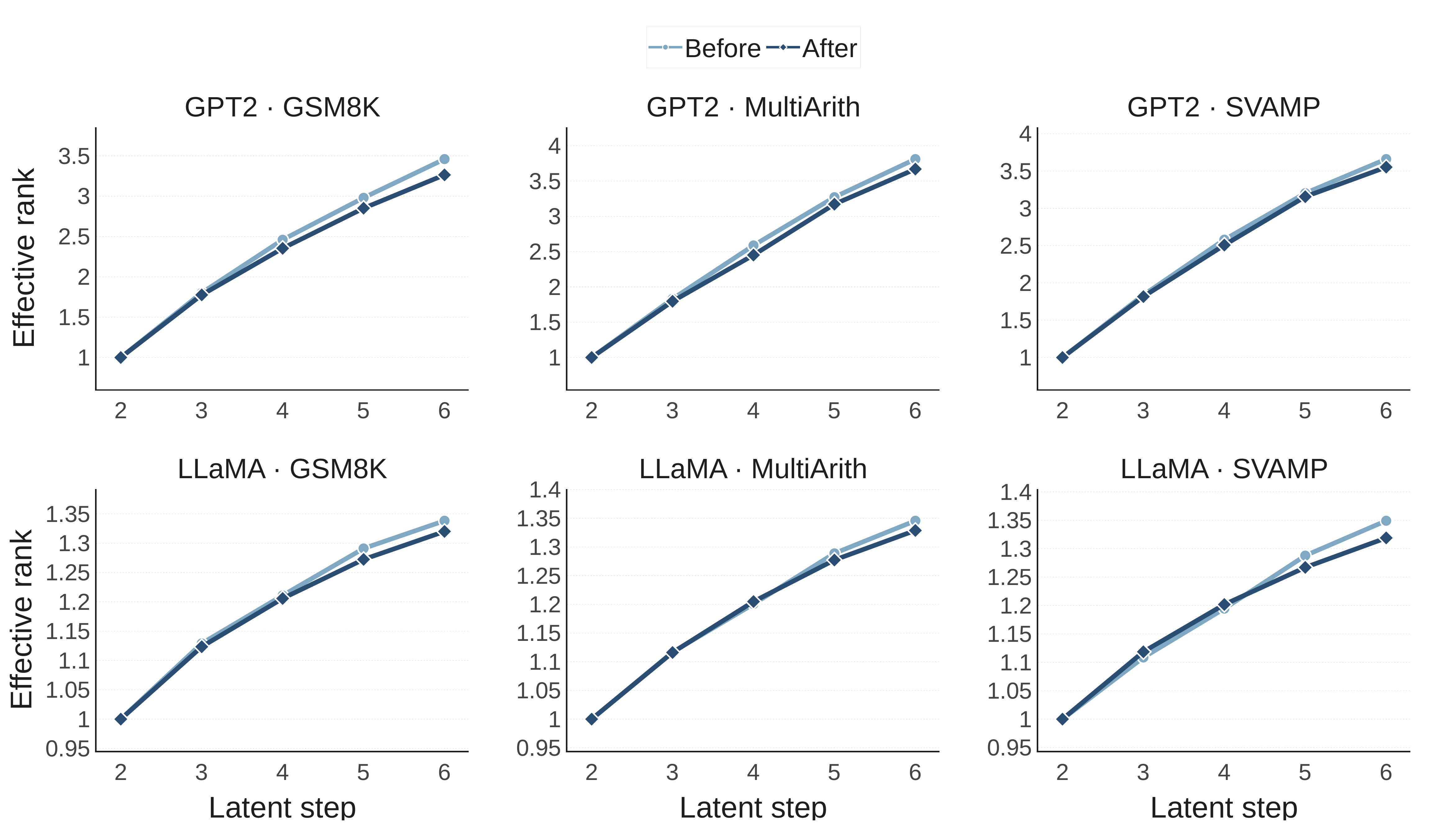}
    \vspace{-5pt}
    \caption{Full prefix effective rank curves pre- and post-SLPO for COCONUT.
    The plot shows \eqref{eq:effective-rank} at prefix length $t$,
    averaged over the first $32$ problems in each dataset.
    Curves are grouped by backbone and dataset.}
    \label{fig:appendix-rank-coconut}
    \vspace{-8pt}
\end{figure}

\begin{figure}[t]
    \centering
    \includegraphics[width=0.95\linewidth]{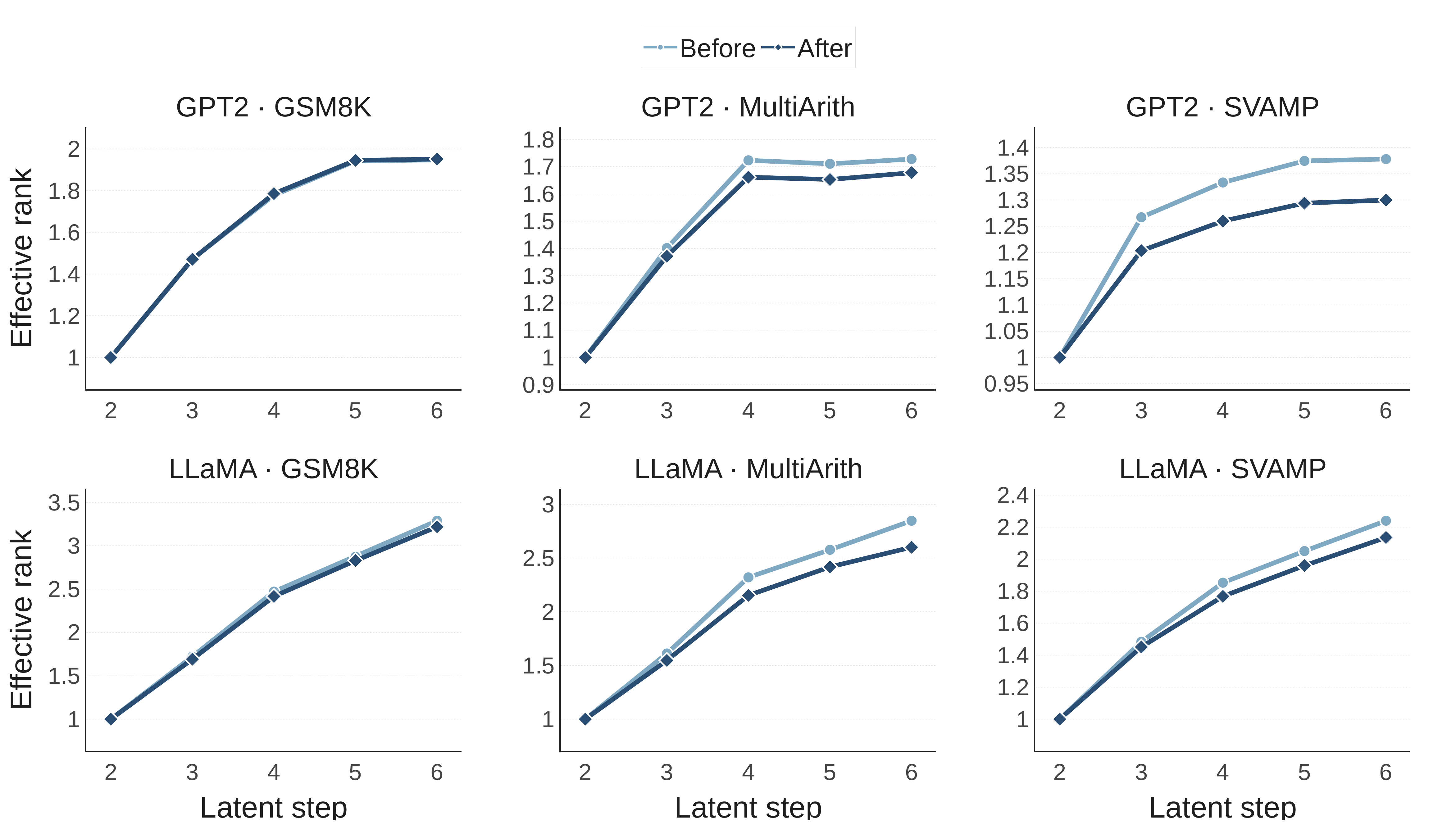}
    \vspace{-5pt}
    \caption{Full prefix effective rank curves pre- and post-SLPO for CODI.
    The plot shows \eqref{eq:effective-rank} at prefix length $t$,
    averaged over the first $32$ problems in each dataset.
    Curves are grouped by backbone and dataset;
    the main-text rank panels (Fig.~\ref{fig:latent-geometry}) show two representative cases from this figure.}
    \label{fig:appendix-rank-codi}
    \vspace{-8pt}
\end{figure}

Figs.~\ref{fig:appendix-rank-coconut} and~\ref{fig:appendix-rank-codi} extend the rank analysis to all dataset--backbone combinations for COCONUT and CODI.
The main-text conclusion follows the same pre- versus post-SLPO comparison at fixed $t$:
the post-SLPO curve lies lower while the progression shape across latent steps is preserved.
\section{Further Analysis}

\subsection{Hyperparameter Sensitivity}
\label{app:rollout-hyperparam}
\label{app:hyperparam-sweeps}

SLPO couples surrogate transition-likelihood estimation with reward optimization through rollout hyperparameters $K$ and $G$.
$K$ controls the number of stochastic forward passes used to estimate each surrogate in Sec.~\ref{sec:slpo},
whereas $G$ is the group size controlling how many complete trajectories are sampled per problem.
We vary each factor over $\{2,4,8\}$ while holding the other fixed (Fig.~\ref{fig:param-search-pass2});
full values appear in App.~\ref{app:hyperparam-sweeps}.

\begin{figure}[t]
    \centering
    \begin{subfigure}{0.24\linewidth}
        \centering
        \includegraphics[width=\linewidth]{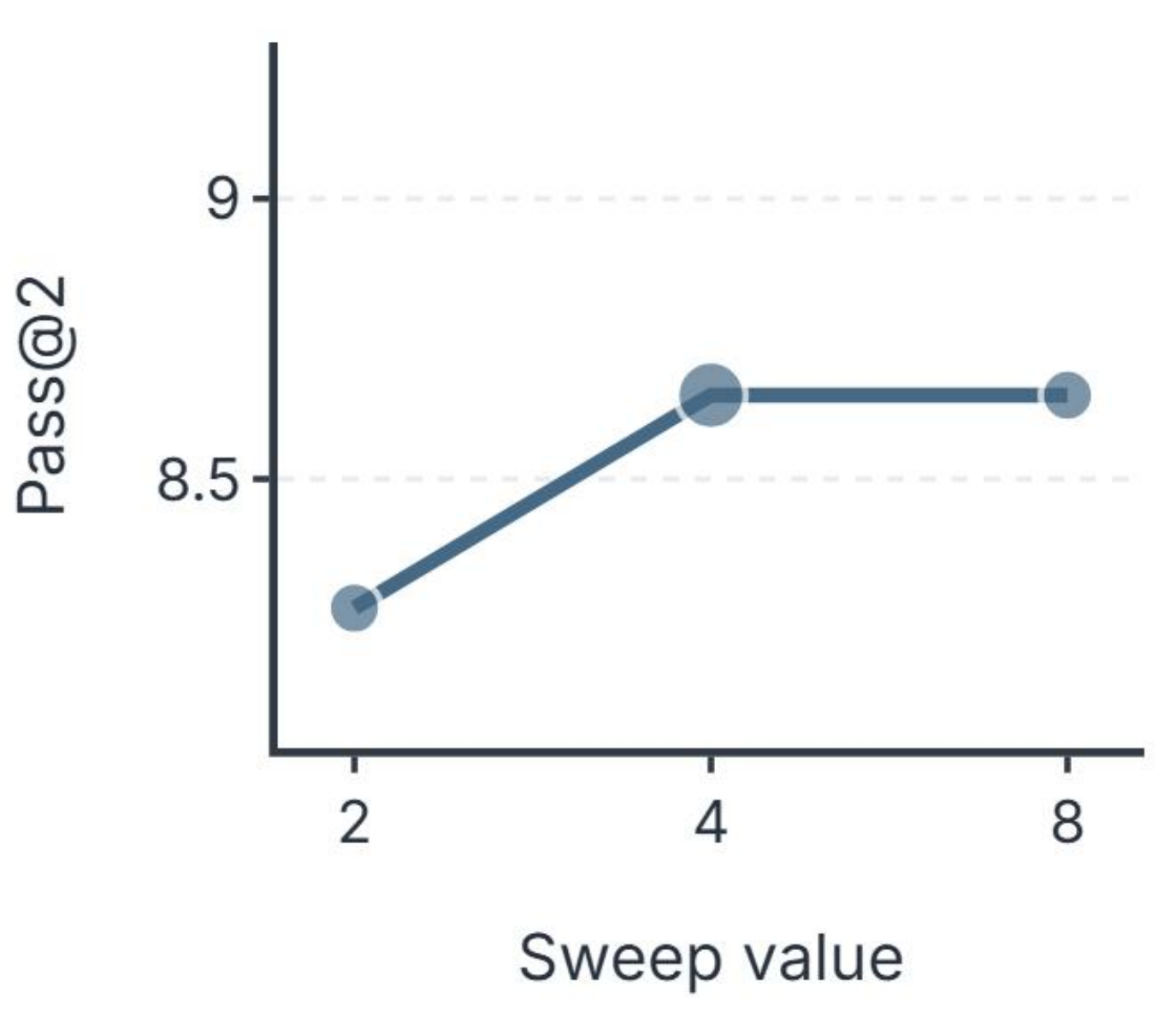}
        \vspace{-14pt}
        \caption{$K$ sweep, GSM-Hard}
        
        \label{fig:param-search-k-gsmhard}
    \end{subfigure}
    \hfill
    \begin{subfigure}{0.24\linewidth}
        \centering
        \includegraphics[width=\linewidth]{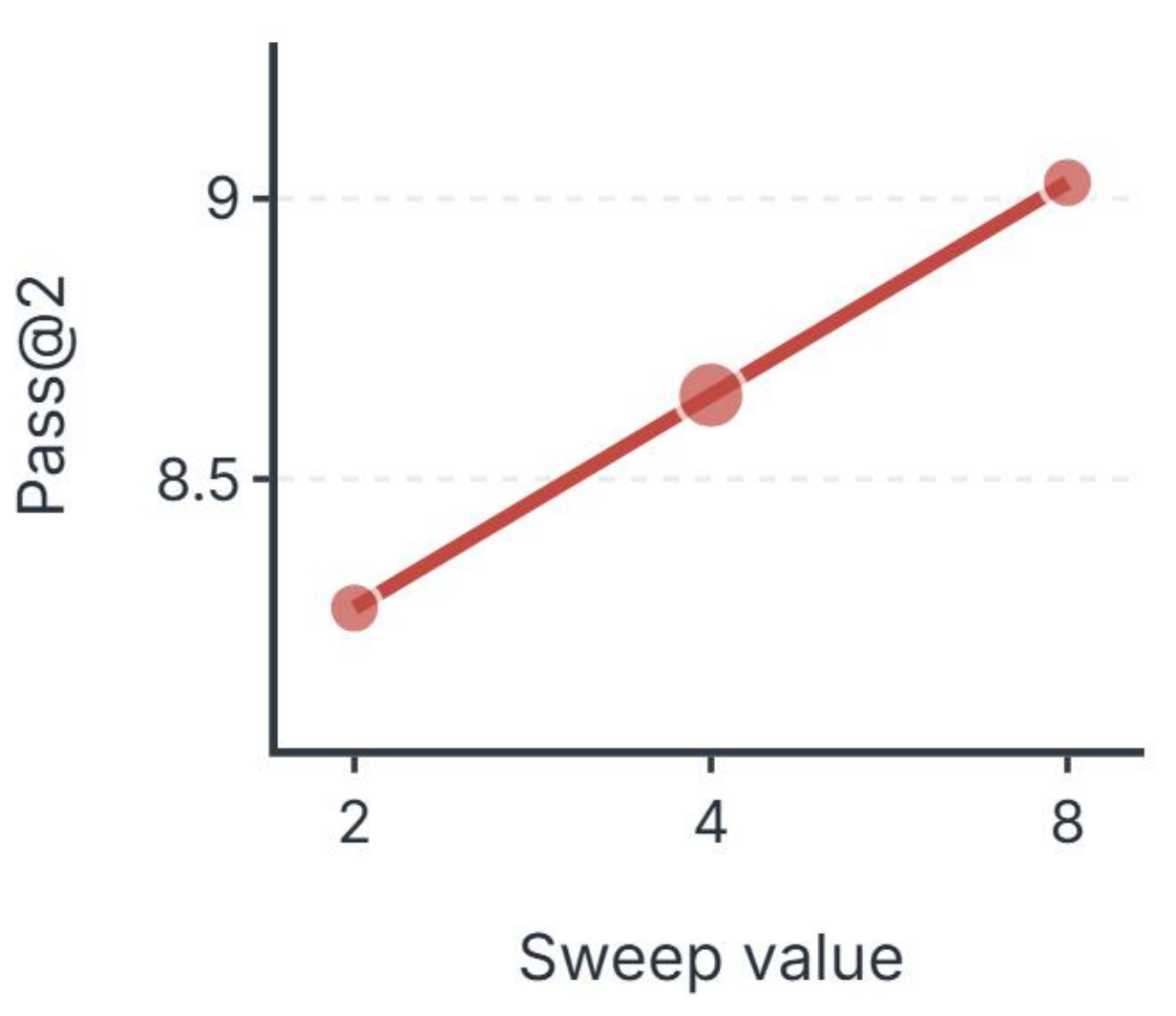}
        \vspace{-14pt}
        \caption{$G$ sweep, GSM-Hard}
        
        \label{fig:param-search-g-gsmhard}
    \end{subfigure}
    \hfill
    \begin{subfigure}{0.24\linewidth}
        \centering
        \includegraphics[width=\linewidth]{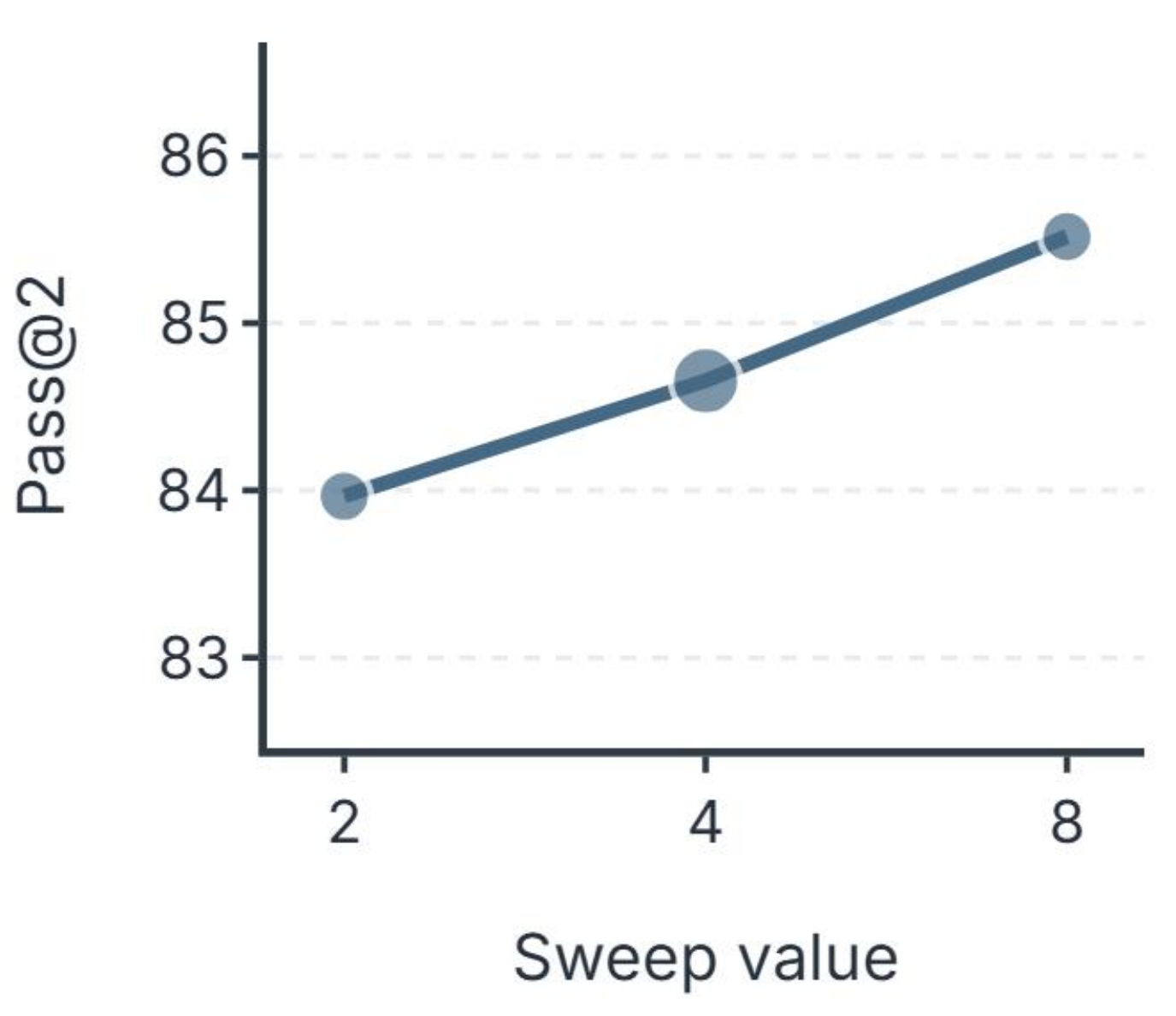}
        \vspace{-14pt}
        \caption{$K$ sweep, MultiArith}
        
        \label{fig:param-search-k-multiarith}
    \end{subfigure}
    \hfill
    \begin{subfigure}{0.24\linewidth}
        \centering
        \includegraphics[width=\linewidth]{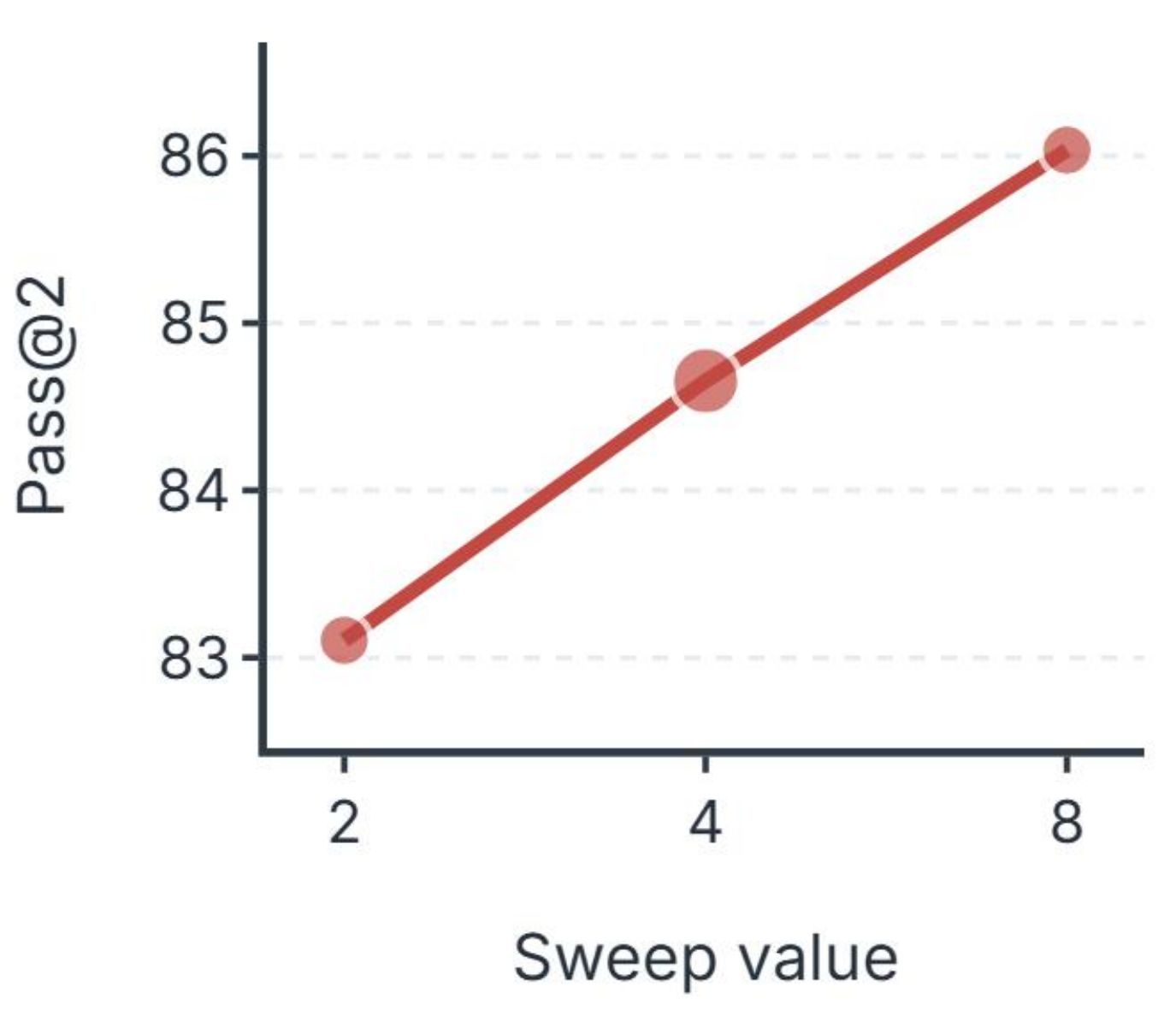}
        \vspace{-14pt}
        \caption{$G$ sweep, MultiArith}
        \label{fig:param-search-g-multiarith}
    \end{subfigure}
    \caption{Controlled hyperparameter sweeps for SLPO.
    Each panel varies one factor over $\{2,4,8\}$ while holding the other fixed ($K$ sweep at $G=4$; $G$ sweep at $K=4$).
    Full values are in App.~\ref{app:hyperparam-sweeps}.}
    \label{fig:param-search-pass2}
    \vspace{-8pt}
\end{figure}

Fig.~\ref{fig:param-search-pass2} identifies group size $G$ as the primary rollout-scaling axis:
increasing $G$ consistently improves Pass@2 on both GSM-Hard and MultiArith.
Performance is substantially less sensitive to $K$,
indicating that a modest number of stochastic forward passes already provides an effective surrogate estimate.
SLPO can therefore direct additional rollout compute toward exploring more complete trajectories,
where it yields the clearest return.

This appendix reports the full Pass@2 values underlying Fig.~\ref{fig:param-search-pass2}.
We vary one rollout hyperparameter over $\{2,4,8\}$ while holding the other fixed at $K=4$ or $G=4$.

\begin{figure}[h]
    \centering
    \begin{subfigure}{0.38\linewidth}
        \centering
        \includegraphics[width=\linewidth]{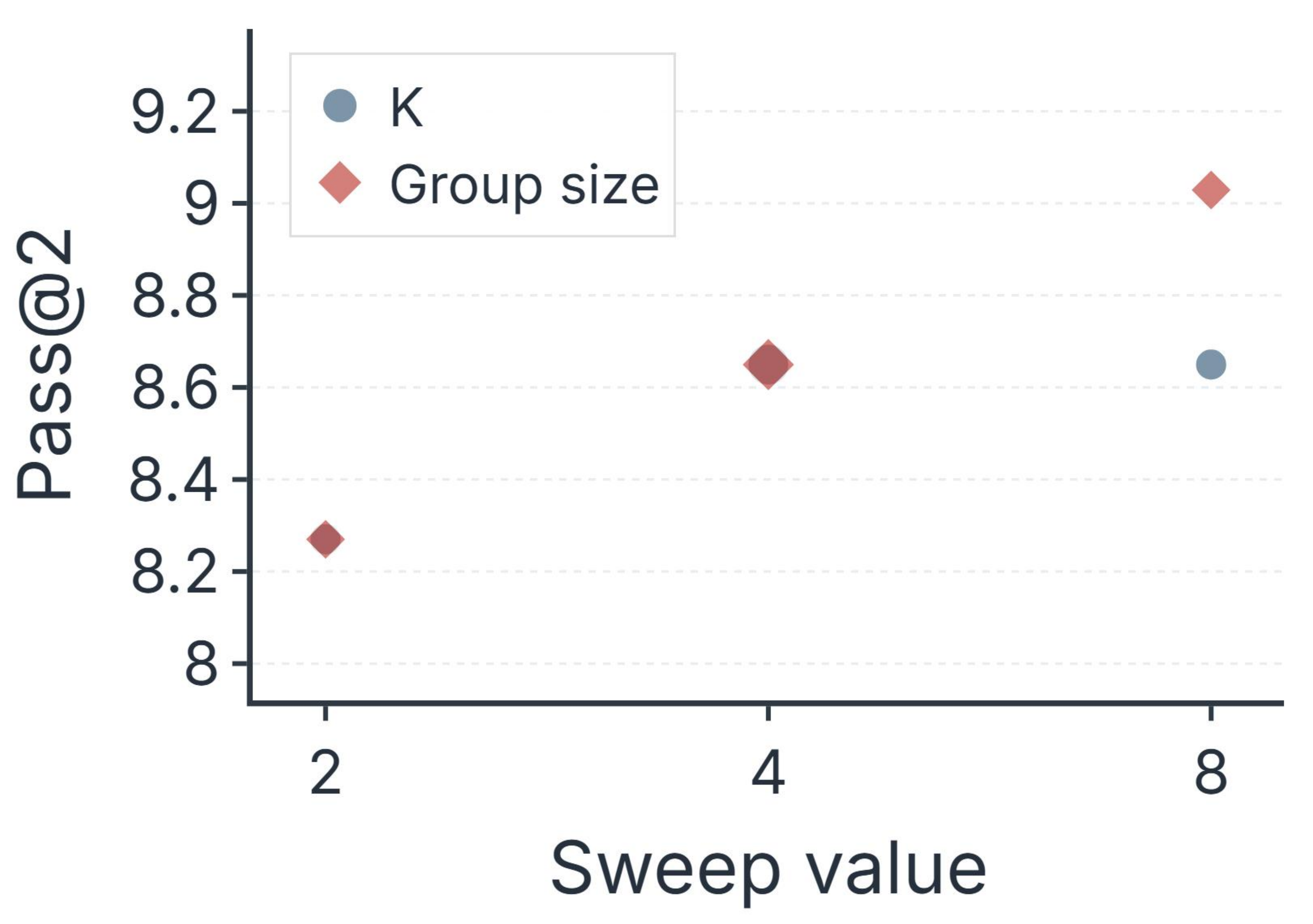}
        \vspace{-14pt}
        \caption{GSM-Hard}
    \end{subfigure}
    \begin{subfigure}{0.38\linewidth}
        \centering
        \includegraphics[width=\linewidth]{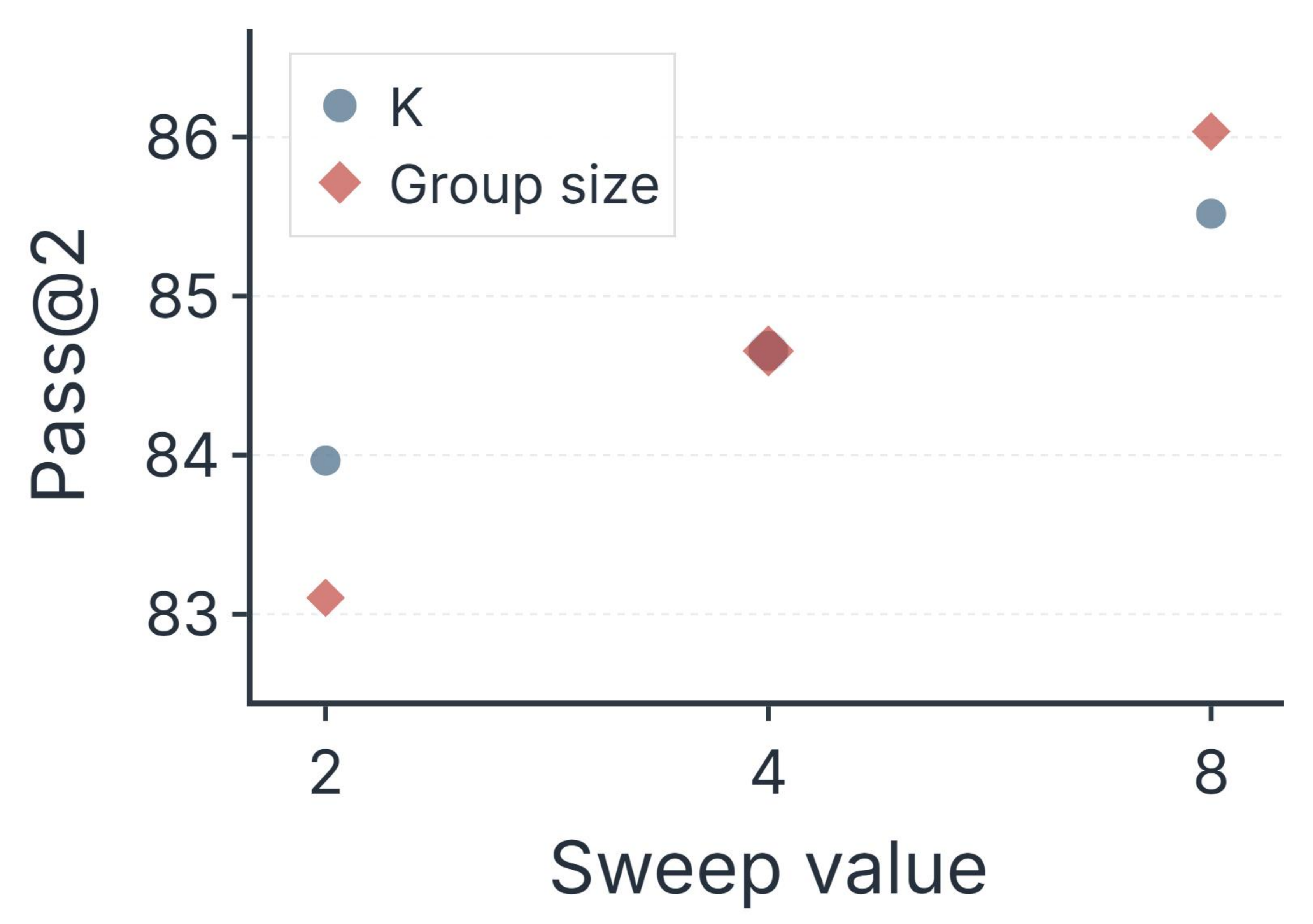}
        \vspace{-14pt}
        \caption{MultiArith}
    \end{subfigure}
    \caption{Full Pass@2 sweeps for $K$ (with $G=4$ fixed) and $G$ (with $K=4$ fixed).
    The larger marker at $4$ marks the value used in that one-factor sweep.}
    \label{fig:param-search-pass2-appendix}
    \vspace{-8pt}
\end{figure}

\subsection{Computational Profile}
\label{app:computational-cost}

SLPO introduces training-time computation through $G$ sampled trajectories
and $K$ stochastic forwards for estimating each surrogate transition likelihood.
We profile this cost on the same RTX 5880 Ada GPU class using two unique prompts,
with two warmup steps followed by eight measured optimizer steps.
Tab.~\ref{tab:computational-cost} reports mean wall-clock time per step
and its rollout, surrogate-construction, and backward components.
The profile isolates the SLPO training stage;
the one-time stopping-gate cold start is described separately in App.~\ref{app:stop-gate-cold-start}.

\begin{table}[h]
    \centering
    \caption{Per-step computational profile on RTX 5880 Ada.
    All configurations use MC-dropout.
    Relative time is normalized by the paper configuration $(K,G)=(4,8)$.
    \textbf{Sur./step} is the fraction of total step time spent constructing the surrogate likelihood.}
    \label{tab:computational-cost}
    \begingroup
    \setlength{\tabcolsep}{4.2pt}
    \renewcommand{\arraystretch}{1.08}
    \resizebox{.7\linewidth}{!}{%
    \begin{tabular}{cccccccc}
        \toprule
        $K$ & $G$ & Step (s) & Rel. & Rollout (s) & Surrogate (s) & Backward (s) & Sur./step \\
        \midrule
        1 & 8 & 0.624 & $0.55\times$ & 0.321 & 0.108 & 0.177 & 17\% \\
        \rowcolor{Blue!6}\textbf{4} & \textbf{8} & \textbf{1.130} & $\mathbf{1.00\times}$ & 0.294 & 0.274 & 0.545 & 24\% \\
        8 & 8 & 2.021 & $1.79\times$ & 0.319 & 0.585 & 1.101 & 29\% \\
        4 & 4 & 0.676 & $0.60\times$ & 0.243 & 0.146 & 0.272 & 22\% \\
        \bottomrule
    \end{tabular}
    }
    \endgroup
\end{table}

Under the paper configuration $(K,G)=(4,8)$,
an optimizer step takes $1.130$ seconds;
surrogate construction accounts for $0.274$ seconds,
or $24\%$ of the total,
while rollout generation and the backward pass account for $0.294$ and $0.545$ seconds.
The profile also exposes a direct compute control:
reducing the group size from $G=8$ to $G=4$ at $K=4$ lowers step time to $0.676$ seconds ($0.60\times$),
whereas increasing $K$ from $4$ to $8$ at $G=8$ raises it to $2.021$ seconds ($1.79\times$).
Together with the hyperparameter analysis in Fig.~\ref{fig:param-search-pass2},
which finds larger returns from rollout diversity than from additional surrogate samples,
these measurements support allocating training compute to $G$ while keeping $K$ moderate.

The surrogate estimator and the $K$ stochastic forwards are used only during training.
At inference,
SLPO performs the original latent recurrence together with the learned stopping head;
its compute is therefore governed by the realized adaptive trajectory length rather than by $K$ or $G$.

\section{Surrogate Optimization Guarantees}
\label{app:surrogate-objective}

We establish five properties of the surrogate in Sec.~\ref{sec:slpo}:
its advantage-weighted update,
the empirical objective implemented by SLPO,
its relation to the expected-reward gradient,
a sufficient condition for expected-reward ascent,
and the regularity induced by its variance floor.
Throughout,
write $h$ for a stop-gradient latent state sampled by MC-dropout,
$(\mu_\theta,\sigma_\theta^{2})$ for the isotropic moments recomputed from $K$ dropout forwards,
and
\[
\log\widetilde{\pi}_\theta(h)
=
-\tfrac{d}{2}\log(2\pi\sigma_\theta^{2})
-
\tfrac{\|h-\mu_\theta\|_{2}^{2}}{2\sigma_\theta^{2}}.
\]
Answer-token and stopping terms are unchanged from Sec.~\ref{sec:slpo}
and are omitted from the latent-only statements below.

\begin{proposition}[Mean-gradient component under fixed variance]
\label{prop:aw-matching}
Fix $\sigma_\theta^{2}>0$
(or treat $\sigma_\theta^{2}$ as stop-gradient),
and hold $(h,A)$ fixed.
Then
\[
-A\log\widetilde{\pi}_\theta(h)
=
\frac{A}{2\sigma_\theta^{2}}
\|h-\mu_\theta\|_{2}^{2}
+
C(A,\sigma_\theta^{2}),
\]
so
\[
-A\,\nabla_\theta\log\widetilde{\pi}_\theta(h)
=
\frac{A}{\sigma_\theta^{2}}
\,
(\mu_\theta-h)^{\top}
\nabla_\theta\mu_\theta
\]
whenever gradients are taken only through $\mu_\theta$.
A positive advantage therefore pulls $\mu_\theta$ toward $h$,
and a negative advantage pushes it away.
When $\sigma_\theta^{2}$ is also trainable,
this identity describes only the mean-gradient component of the latent update;
the variance path is additional.
\end{proposition}

\begin{proof}[Proof sketch]
The isotropic surrogate log-density expands as
$\log\widetilde{\pi}_\theta(h)=-\|h-\mu_\theta\|_{2}^{2}/(2\sigma_\theta^{2})+c(\sigma_\theta^{2})$.
Multiplying by $-A$ gives the stated loss identity.
Differentiating in $\mu_\theta$ and applying the chain rule yields the gradient form.
\end{proof}

\begin{proposition}[Detached empirical surrogate loss]
\label{prop:surr-pg}
Draw a frozen rollout batch
$\{\xi_i\}_{i=1}^{G}$ from MC-dropout under parameters $\bar\theta$,
compute stop-gradient advantages $A_i=\mathrm{sg}(R_i-b_i)$,
and define the empirical surrogate loss
\[
\widehat{\mathcal{L}}_{\mathrm{surr}}(\theta;\bar\theta)
=
-
\frac{1}{G}
\sum_{i=1}^{G}
A_i
\widetilde{\ell}_\theta(\xi_i\mid x).
\]
Then the implemented gradient of the latent, answer, and gate terms in Sec.~\ref{sec:slpo}
coincides with
$\nabla_\theta\widehat{\mathcal{L}}_{\mathrm{surr}}(\theta;\bar\theta)$
evaluated at $\theta=\bar\theta$,
with the sampled latent states in $\widetilde{\ell}_\theta$ treated as stop-gradient targets.
Thus the implementation optimizes an explicit reward-weighted surrogate score on a detached batch,
establishing the empirical optimization target of SLPO.
\end{proposition}

\begin{proof}[Proof sketch]
After sampling,
$(\xi_i,A_i)$ no longer depend on the differentiation variable $\theta$.
Differentiating $\widehat{\mathcal{L}}_{\mathrm{surr}}$ therefore passes only through the scored factors
$\widetilde{\ell}_\theta(\xi_i\mid x)$,
which is precisely the implemented backward pass.
\end{proof}

\begin{proposition}[Gradient approximation under score alignment]
\label{prop:surr-score-mismatch}
Let $q_\theta(\xi\mid x)$ denote a behavior trajectory law that admits a differentiable score,
and suppose the expected-reward gradient admits the score representation
\[
g_J
=
\mathbb{E}_{\xi\sim q_\theta}
\left[
A(\xi)\nabla_\theta\log q_\theta(\xi\mid x)
\right]
\]
for a square-integrable advantage $A(\xi)$.
Define the corresponding population surrogate direction
\[
g_{\mathrm{surr}}
=
\mathbb{E}_{\xi\sim q_\theta}
\left[
A(\xi)\nabla_\theta\widetilde{\ell}_\theta(\xi\mid x)
\right].
\]
If
\[
\mathbb{E}_{\xi\sim q_\theta}
\left[
\left\|
\nabla_\theta\widetilde{\ell}_\theta(\xi\mid x)
-
\nabla_\theta\log q_\theta(\xi\mid x)
\right\|_2^2
\right]
\leq
\varepsilon_{\mathrm{score}}^2,
\]
then
\[
\left\|g_{\mathrm{surr}}-g_J\right\|_2
\leq
\sqrt{\mathbb{E}_{q_\theta}[A(\xi)^2]}
\,
\varepsilon_{\mathrm{score}}.
\]
\end{proposition}

\begin{proof}[Proof sketch]
Write the gradient difference as the expectation of the advantage multiplied by the score mismatch.
Cauchy--Schwarz gives the stated bound.
\end{proof}

\begin{proposition}[Local expected-reward ascent]
\label{prop:surr-reward-ascent}
Suppose $J$ has an $L$-Lipschitz gradient on a neighborhood of $\theta$,
and let $g_{\mathrm{surr}}$ be an ascent direction obtained from the population surrogate objective.
If
\[
\left\|g_{\mathrm{surr}}-\nabla_\theta J(\theta)\right\|_2
\leq
\delta,
\]
Then for any $\eta>0$ such that the segment
$\{\theta+s\eta g_{\mathrm{surr}}:s\in[0,1]\}$
lies in this neighborhood,
\[
\begin{aligned}
J(\theta+\eta g_{\mathrm{surr}})
\geq\;&
J(\theta)
+
\eta
\left(
\|g_{\mathrm{surr}}\|_2^2
-
\delta\|g_{\mathrm{surr}}\|_2
\right)
\\
&-
\frac{L\eta^2}{2}
\|g_{\mathrm{surr}}\|_2^2.
\end{aligned}
\]
Consequently,
if $\|g_{\mathrm{surr}}\|_2>\delta$ and
\[
0
<
\eta
<
\frac{2(\|g_{\mathrm{surr}}\|_2-\delta)}
{L\|g_{\mathrm{surr}}\|_2},
\]
the surrogate step strictly increases the expected reward $J$.
\end{proposition}

\begin{proof}[Proof sketch]
$L$-smoothness gives
\[
J(\theta+\eta g_{\mathrm{surr}})
\geq
J(\theta)
+
\eta
\langle\nabla_\theta J(\theta),g_{\mathrm{surr}}\rangle
-
\frac{L\eta^2}{2}\|g_{\mathrm{surr}}\|_2^2.
\]
Writing
$\nabla_\theta J(\theta)=g_{\mathrm{surr}}+e$
with $\|e\|_2\leq\delta$
and applying Cauchy--Schwarz yields the result.
\end{proof}

\paragraph{Role of MC-dropout.}
MC-dropout supplies the exploratory behavior trajectories,
while a finite-$K$ moment-matched Gaussian score turns those trajectories into a differentiable optimization interface.
Larger $K$ reduces Monte Carlo error in the estimated moments,
and Prop.~\ref{prop:surr-score-mismatch} shows how score alignment controls the population-gradient approximation.
Prop.~\ref{prop:surr-reward-ascent} then converts this approximation into an explicit condition for expected-reward improvement.

\begin{proposition}[Variance-floor regularity]
\label{prop:sigma-floor}
For fixed $h$ and a variance floor $\epsilon>0$,
\[
\sup_{\mu,\,\sigma^{2}\geq\epsilon}
\log\widetilde{\pi}(h;\mu,\sigma^{2})
=
-\tfrac{d}{2}\log(2\pi\epsilon),
\]
attained at $\mu=h$ and $\sigma^{2}=\epsilon$.
For every fixed $\sigma^{2}\geq\epsilon$,
$\mu=h$ is the unique maximizing mean.
The floor is essential for the finite optimum:
over $\sigma^{2}>0$,
\[
\sup_{\mu,\sigma^{2}>0}
\log\widetilde{\pi}(h;\mu,\sigma^{2})
=
+\infty,
\]
because the supremum diverges as $\sigma^{2}\to 0^{+}$ along $\mu=h$.
\end{proposition}

\begin{proof}[Proof sketch]
Completing the square in the isotropic Gaussian density shows that,
for fixed $\sigma^{2}$,
the maximum over $\mu$ is attained uniquely at $\mu=h$
and equals $-\tfrac{d}{2}\log(2\pi\sigma^{2})$.
The variance floor maximizes this expression at $\sigma^{2}=\epsilon$,
giving the finite joint optimum above.
Removing the floor and sending $\sigma^{2}\to0^{+}$ along $\mu=h$ yields the stated divergence.
\end{proof}

\paragraph{Implications.}
Prop.~\ref{prop:aw-matching} identifies the mean-matching force in the latent update when variance is held fixed.
Prop.~\ref{prop:surr-pg} characterizes the implemented step as gradient descent on a detached,
reward-weighted surrogate score.
Prop.~\ref{prop:surr-score-mismatch} links score alignment to population-gradient approximation,
and Prop.~\ref{prop:surr-reward-ascent} gives a sufficient local condition for expected-reward improvement.
RLOO and GRPO instantiate the stop-gradient advantage $A_i$ within this shared interface.
Prop.~\ref{prop:sigma-floor} establishes the finite optimum induced by the variance floor $\epsilon$.

\end{document}